\documentclass[preprint,12pt,3p]{elsarticle}

\usepackage{amssymb}
\usepackage{amsmath}
\usepackage{mathtools}
\usepackage{array}
\usepackage{cancel}
\usepackage{bm}
\usepackage{bbm}
\usepackage[usenames, dvipsnames]{color}
\usepackage{algorithm,algpseudocode}
\usepackage[toc,page]{appendix}
\usepackage{caption}
\usepackage{subcaption}
\usepackage{hyperref}
\usepackage{multicol}
\usepackage{textcomp}
\usepackage[hang,flushmargin]{footmisc}
\usepackage{siunitx}
\usepackage{mdframed}
\usepackage{listings}
\usepackage{sidecap}

\newcommand{\bigCI}{\perp\!\!\!\perp}
\newcommand{\varA}[1]{{\operatorname{#1}}}
\algnewcommand{\LeftComment}[1]{\Statex \hspace{0.4cm} \(\triangleright\) #1}
\algnewcommand\algorithmicforeach{\textbf{for each}}
\algdef{S}[FOR]{ForEach}[1]{\algorithmicforeach\ #1\ \algorithmicdo}

\makeatletter
\providecommand{\doi}[1]{%
  \begingroup
    \let\bibinfo\@secondoftwo
    \urlstyle{rm}%
    \href{http://dx.doi.org/#1}{%
      doi:\discretionary{}{}{}%
      \nolinkurl{#1}%
    }%
  \endgroup
}
\makeatother

\newcommand{\BlackBox}{\rule{1.5ex}{1.5ex}}  
\newenvironment{proof}{\par\noindent{\bf Proof\ }}{\hfill\BlackBox\\[2mm]}
 
\newtheorem{theorem}{Theorem}

\makeatletter
\newenvironment{breakablealgorithm}
  {
   \begin{center}
     \refstepcounter{algorithm}
     \hrule height.8pt depth0pt \kern2pt
     \renewcommand{\caption}[2][\relax]{
       {\raggedright\textbf{\ALG@name~\thealgorithm} ##2\par}%
       \ifx\relax##1\relax 
         \addcontentsline{loa}{algorithm}{\protect\numberline{\thealgorithm}##2}%
       \else 
         \addcontentsline{loa}{algorithm}{\protect\numberline{\thealgorithm}##1}%
       \fi
       \kern2pt\hrule\kern2pt
     }
  }{
     \kern2pt\hrule\relax
   \end{center}
  }
\makeatother

\usepackage{graphicx}
\graphicspath{{./figures/}}
\DeclareGraphicsExtensions{.pdf,.jpeg,.png}

\journal{Information Sciences}

\begin{document}

\begin{frontmatter}

\title{Copula Index for Detecting Dependence and Monotonicity between Stochastic Signals}

\author[label1]{Kiran Karra\corref{cor1}}
\address[label1]{Virginia Tech \\ Department of Electrical and Computer Engineering \\ 900 N. Glebe Rd., Arlington, VA, 22031}

\cortext[cor1]{Corresponding Author}
\ead{kiran.karra@vt.edu}
\ead[url]{kirankarra.wordpress.com}

\author[label1]{Lamine Mili}
\ead{lmili@vt.edu}

\begin{abstract}
This paper introduces a nonparametric copula-based index for detecting the strength and monotonicity structure of linear and nonlinear statistical dependence between pairs of random variables or stochastic signals.  Our index, termed Copula Index for Detecting Dependence and Monotonicity (\textit{CIM}), satisfies several desirable properties of measures of association, including most of R\'enyi's properties, the data processing inequality (DPI), and consequently self-equitability.  Synthetic data simulations reveal that the statistical power of \textit{CIM} compares favorably to other state-of-the-art measures of association that are proven to satisfy the DPI.  Simulation results with real-world data reveal \textit{CIM}'s unique ability to detect the monotonicity structure among stochastic signals to find interesting dependencies in large datasets.  Additionally, simulations show that \textit{CIM} shows favorable performance to estimators of mutual information when discovering Markov network structure.
\end{abstract}


\end{frontmatter}

\section{Introduction}
A fundamental problem in exploratory data analysis involves understanding the organization and structure of large datasets.  An unsupervised approach to this problem entails modeling the features within these datasets as random variables and discovering the dependencies between them using measures of association.  Many measures of association have been introduced in the literature, including the correlation coefficient \cite{corrcoef}, \textit{MIC} \cite{mic}, the \textit{RDC} \cite{rdc}, the \textit{dCor} \cite{dcorr}, the \textit{Ccor} \cite{ccorr_cite}, and \textit{CoS} \cite{cos_cite}.  In addition, many estimators of mutual information such as the \textit{kNN} \cite{knn_mi}, the \textit{vME} \cite{vonMisesMI}, and the \textit{AP} \cite{shannonapMI} are used as measures of association.

However, properties of the dataset such as whether the data are discrete or continuous, linear or nonlinear, monotonic or nonmonotonic, noisy or not, and independent and identically distributed (\textit{i.i.d.}) or not, to name a few, are important factors to consider when deciding which measure(s) of association one may use when performing exploratory data analysis.  Because these properties are typically not known a priori, the task of selecting a single measure of association is difficult.  Additionally, a measure of association should also satisfy certain desirable properties, namely R\'enyi's properties \cite{Renyi1959}, the data processing inequality (DPI) \cite{mic_not_equitable}, and equitability \cite{mic}.  However, no measure satisfies all these properties while simultaneously being able to handle the different types and properties of data described.  For example, the most commonly used measure of statistical dependence, the correlation coefficient, only measures linear dependence.  Others such as the \textit{RDC} exhibit high bias and relatively weak statistical power for the basic (and arguably the most important \cite{mic_not_equitable,simonandtibs}) linear dependency structure, due to overfitting.  Finally, estimators of mutual information do not have a theoretical upper bound, meaning that the values can only be used in a relative sense.  Even though each of the aforementioned measures of association perform well in the conditions for which they were designed, they cannot be used as an omnibus solution to an exploratory data analysis problem.

To help address these shortcomings, we introduce a new index of nonlinear dependence, \textit{CIM}. This index is based on copulas and the rank statistic Kendall's $\tau$ \cite{kendalltau}, that naturally handles linear and nonlinear associations between continuous, discrete, and hybrid random variables (pairs of random variables where one is continuous and the other is discrete) or stochastic signals.  Additionally, \textit{CIM} provides good statistical power over a wide range of dependence structures and satisfies several desirable properties of measures of association including most of R\'enyi's properties and the data processing inequality.  Furthermore, it uniquely identifies regions of monotonicity in the dependence structure which provide insight into how the data should be modeled stochastically.  Due to these properties, \textit{CIM} is a powerful tool for exploratory data analysis.

This paper is organized as follows.  Section \ref{sec:ktauhat} introduces copulas, rank statistics, and modifying Kendall's $\tau$ to account for discrete and hybrid random variables.  Section \ref{sec:cim} introduces \textit{CIM} index, which builds upon the extension of Kendall's $\tau$ in Section \ref{sec:ktauhat} to handle both monotonic and non-monotonic dependencies and proposes an algorithm to estimate it.  Here, important properties which theoretically ground \textit{CIM} as a desirable measure of association are proved.  Additionally, an estimation algorithm and it's properties are discussed and it is shown that the algorithm is robust to hyperparameter selection.  Next, Section \ref{sec:exp} provides simulations to exercise the developed metric against other state-of-the-art dependence metrics, including \textit{MIC}, the \textit{RDC}, the \textit{dCor}, the \textit{Ccor}, and \textit{CoS} and measures of information including \textit{kNN}, \textit{vME}, and \textit{AP} using synthetic data.  These simulations reveal that \textit{CIM} compares favorably to other measures of association that satisfy the Data Processing Inequality (DPI).  Simulations with real-world data show how \textit{CIM} can be used for many exploratory data analysis and machine learning applications, including probabilistic modeling, discovering interesting dependencies within large datasets, and Markov network discovery.  These simulations show the importance of considering the monotonicity structure of data when performing probabilistic modeling, a property that only \textit{CIM} can measure.  The favorability of using \textit{CIM} when performing Markov network discovery is shown through the \textbf{netbenchmark} simulation framework.  Concluding remarks are then provided in Section \ref{sec:conclusion}.

\section{Copulas, Concordance, and Rank Statistics} \label{sec:ktauhat}
In this section, we provide a brief overview of copulas, concordance, and rank statistics.  We focus on Kendall's $\tau$  as it provides the basis for \textit{CIM}, and propose a new extension of Kendall's $\tau$ to account for hybrid random variables.  The motivation for this extension comes from the need to assess the strength of association between hybrid random variables, both when exploring real-world datasets where data can consist of both continuous and discrete data simultaneously, and in feature selection as in \textit{mRMR} \cite{mrmr}, where a necessary step is to assess the strength of association between continuous features and discrete classes.  Properties of this extension, denoted $\tau_{KL}$, are then highlighted and discussed.

\subsection{Introduction to Copulas and Concordance for Continuous Random Variables}
Copulas are multivariate joint probability distribution functions for which the marginal distributions are uniform \cite{nelsen}.  In the bivariate case, the existence of a copula $C$ associated with the random variables, $X$ and $Y$, following a joint cumulative distribution function, $H$, and marginals cumulative distributions $F_X$ and $G_Y$, respectively, is ensured by Sklar's theorem, which states that

\begin{equation} \label{eq:sklar1}
H(x,y) = C(F_X(x), G_Y(y)).
\end{equation}
This theorem guarantees the unicity of the copula $C$ for continuous random variables and it unveils its major property, which is its ability to capture the unique dependency structure between any random variables $X$ and $Y$.  Thus, the copula $C$ can be used to define a measure of dependence between continuous random variables.  

A popular measure of dependence that is based on the copula is concordance, which measures the degree to which two random variables are monotonically associated with each other.  More precisely, points in $\mathbb{R}^2$, $(x_i, y_i)$ and $(x_j, y_j)$, are concordant if $(x_i - x_j) (y_i - y_j) > 0$ and discordant if $(x_i - x_j) (y_i - y_j) < 0$ \cite{nelsen}.  This can be probabilistically represented by the concordance function, $Q$, defined as

\begin{align} 
Q &= P[(X_1-X_2)(Y_1-Y_2)>0] - P[(X_1-X_2)(Y_1-Y_2)<0] \label{eq:qfunc} \\
  & = 4 \int \int_{\mathbf{I}^2} C_2(u,v) dC_1(u,v) - 1 \label{eq:qfunc_copula}
\end{align}
where $(X_1,Y_1)$ and $(X_2,Y_2)$ are independent vectors of continuous random variables with distribution functions $H_1$ and $H_2$ having common margins of $F$ (of $X_1$ and $X_2$) and $G$ (of $Y_1$ and $Y_2$), and $C_1$ and $C_2$ are copulas of $(X_1,Y_1)$ and $(X_2,Y_2)$, respectively.

Many metrics of association are based on the concept of concordance, with the two most popular being Kendall's $\tau$ \cite{kendalltau} and Spearman's $\rho$ \cite{spearmansrho}.  Kendall's $\tau$ is defined in terms of the concordance function as 
\begin{equation}\label{eq:kendalls_tau_population}
\tau = Q(C,C), 
\end{equation}
where $C$ is the copula of the joint distribution $(X,Y)$, and interpreted as the scaled difference in the probability between concordance and discordance.  It can be estimated by
\begin{equation}\label{eq:kendalls_tau_estimator}
\hat{\tau} = \frac{\text{\# concordant pairs - \# discordant pairs}}{\binom{n}{2}},
\end{equation}
where $n$ is the number of samples.  Concordance-based measures of association such as Kendall's $\tau$ are ideal for detecting linear and nonlinear monotonic dependencies because they are rank statistics.  These measures have the desirable properties of being margin independent and invariant to strictly monotonic transforms of the data \cite{nelsen, scarsini}.  

\subsection{Extension of Kendall's $\tau$ for Hybrid Random Variables}
Although rank statistics work well for measuring monotonic association between continuous random variables, adjustments need to be made to account for discrete and hybrid random variables.  In that case, Sklar's theorem  does not guarantee the unicity of the copula $C$ and many copulas satisfy (\ref{eq:sklar1}) due to ties in the data.  Consequently, the measure of concordance becomes margin-dependent (i.e, cannot be expressed solely in terms the joint distribution's copula as in (\ref{eq:qfunc_copula})) and in many cases cannot reach $+1$ or $-1$ in scenarios of perfect comonotonicity and countermonotonicity, respectively \cite{genest}.

Several proposals for adjusting Kendall's $\tau$ for ties have been made, including $\tau_b$ \cite{taub}, $\tau_{VL}$ \cite{tauvl}, and $\tau_{N}$ \cite{neslehova}.  The common theme among these proposals is that they use different scaling factors to account for ties in the data.  However, even with scaling, perfect monotone dependence does not always imply $|\tau_b|=1$, and $\tau_{VL}$ is not interpretable as a scaled difference between the probabilities of concordance and discordance \cite{genest}.  \citet{neslehova} overcomes both of these limitations and defines the non-continuous version of Kendall's $\tau$, denoted by $\tau_{N}$ \cite[see][Definition 9]{neslehova}.

\begin{equation}\label{eq:taun_theoretical}
    \tau(X_1,X_2) = \frac{4 \int C_\mathbf{X}^S \partial C_\mathbf{X}^S - 1}{\sqrt{(1-E(\Delta F_{X_1}(X_1)))(1-E(\Delta F_{X_2}(X_2)))}},
\end{equation}
where $\mathbf{X} = (X_1,X_2)$ is a bivariate random vector with arbitrary marginals.



\citet{neslehova} then defines an estimator of the non-continuous version of $\tau_N$ as
\begin{equation} \label{eq:taucj}
\hat{\tau}_{N} = \frac{\text{\# concordant pairs - \# discordant pairs}}{\sqrt{\binom{n}{2} - u}\sqrt{\binom{n}{2} - v}} 
\end{equation}
where $u = \sum_{k=1}^r \binom{u_k}{2}$, $v = \sum_{l=1}^s \binom{v_l}{2}$, $r$ is the number of distinct values observed in $x$ and $s$ is the number of distinct values observed in $y$, $u_k$ is the number of times the $k^{th}$ distinct element occurred in the $u$ dimension, $v_l$ is the number of times the $l^{th}$ distinct element occurred in the $v$ dimension.  $\hat{\tau}_{N}$ achieves $+1$ or $-1$ in the comonotonic and countermonotic cases, respectively, for discrete random variables by subtracting the number of ties for each variable $u$ and $v$ independently from the denominator.  The accounting of ties is required due to the strict inequalities used for concordance and discordance in (\ref{eq:qfunc}).  In the continuous case, there are no ties and $\tau_N$ reduces to the original Kendall's $\tau$ defined in (\ref{eq:kendalls_tau_population}). 

\begin{figure}
	\centering
	\begin{minipage}[c]{0.3\textwidth}
        \includegraphics[width=\textwidth]{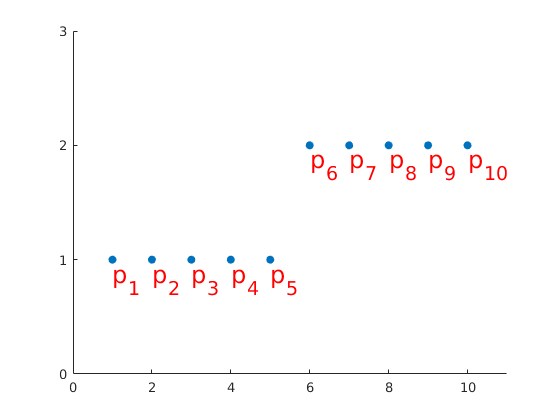}
    \end{minipage}\hfill
   \begin{minipage}[c]{0.6\textwidth}
        \caption{
           The hybrid random variable pair $(X,Y)$ is comonotonic, with $X$ the continuous random variable and $Y$ the discrete random variable.  In the computation of $\hat{\tau}$, the pairs of points $[p_i,p_j]$ for $i=1:5,j=1:5$ and $i=6:10,j=6:10$ are not counted as concordant.  Only the pairs of points $[p_i,p_j]$ for $i=1:5,j=6:10$ are, leading to $\hat{\tau}$ not reaching $+1$ in the perfectly comonotonic case for hybrid random variables.
        } 
        \label{fig:taukl_explanation}
      \end{minipage}
\end{figure}

Although the measure defined by $\tau_N$ is valid for continuous, discrete, or hybrid random variables, the estimator $\hat{\tau}_{N}$ in (\ref{eq:taucj}) does not achieve a value of $+1$ or $-1$ in the perfectly comonotonic and countermonotonic cases, respectively, for hybrid random variables.  In order to make $\hat{\tau}_N$ equal to $+1$ and $-1$ in these cases respectively, we propose to use the maximum number of ties as a correction factor.  This is because in the hybrid case, the numerator of $\hat{\tau}_{N}$ does not count the increasing continuous variables as concordant (or decreasing as discordant).  Fig.~\ref{fig:taukl_explanation} illustrates this counting in an example, and shows why $\hat{\tau}_{N}$ fails to achieve $+1$ or $-1$ in the hybrid random variable case for perfectly comonotonic/countermonotonic random variables respectively.  In it, the pairs of samples along the continuous dimension $x$ within a discrete value ($[p_i,p_j]$ for $i=1:5,j=1:5$ and $i=6:10,j=6:10$) are not counted as comonotonic.  To overcome this drawback, our proposed extension to $\hat{\tau}_{N}$ is defined as

\begin{equation} \label{eq:taukl}
\hat{\tau}_{KL} = 
\begin{cases*}
\frac{\text{\# concordant pairs - \# discordant pairs}}{\binom{n}{2}} & \text{for continuous random variables} \\
\frac{\text{\# concordant pairs - \# discordant pairs}}{\sqrt{\binom{n}{2} - u}\sqrt{\binom{n}{2} -v}} & \text{for discrete random variables} \\
\frac{\text{\# concordant pairs - \# discordant pairs}}{\sqrt{\binom{n}{2} - t}\sqrt{\binom{n}{2} -t}} & \text{for hybrid random variables}
\end{cases*}
\end{equation}
where $t = \mathbf{max}(u,v)-K$, and where $u$ and $v$ are the same as in $\hat{\tau}_{N}$, and $K = \binom{u'}{2} \times v'$, $u'$ denotes the number of overlapping points in the continuous dimension and between different discrete values in the discrete dimension, and $v'$ denotes the number of unique elements in the discrete dimension.  $K$ is zero for perfectly monotonic hybrid random variables, but takes nonzero values for copula-based dependencies; it helps to reduce the bias of $\hat{\tau}_{KL}$ when hybrid random variable samples are drawn from a copula dependency.

The performance of $\hat{\tau}_{KL}$, compared to $\hat{\tau}_b$ and $\hat{\tau}_N$ for perfectly comonotonic random variables is shown in Table~\ref{tab:step}.  It is seen that the proposed modifications to the $\tau_N$ estimate in (\ref{eq:taukl}) do indeed reduce the bias for the hybrid random variables case.  It is also observed that the bias of the $\hat{\tau}_b$ and $\hat{\tau}_N$ is reduced as the number of discrete levels is increased.  However, in all these cases, $\hat{\tau}_{KL}$ still maintains a null bias.  The results in Table~\ref{tab:step} apply for all sample sizes; stated alternatively, the bias is constant across all sample sizes tested. This makes $\hat{\tau}_{KL}$ a compelling alternative to $\hat{\tau}_N$ and $\hat{\tau}_b$.

\begin{table}[ht]
\centering
\begin{tabular}{||c||c|c|c||} 
 \hline
 Discrete Levels & $\hat{\boldsymbol{\tau}}_b$ & $\hat{\boldsymbol{\tau}}_N$ & $\hat{\boldsymbol{\tau}}_{KL}$ \\ [0.5ex] 
 \hline\hline
 2 & 0.58 & 0.71 & 1.00 \\ 
 4 & 0.84 & 0.87 & 1.00 \\
 8 & 0.93 & 0.93 & 1.00 \\
 \hline
\end{tabular}
\caption{Step function dependency with various levels of discretization; it is seen that $\tau$ approaches $1$ as the number of discretization levels increases, but without the bias correction described in (\ref{eq:taukl}), dependence between continuous and discrete random variables is not measured accurately by $\tau_b$ and $\tau_N$.  The results shown here apply for all $M$; stated alternatively, the bias is constant across all sample sizes tested.}
\label{tab:step}
\end{table}

Figs.~\ref{fig:ktauhat_bias_properties} (a),(b),(c), and (d) show the bias and variance between the estimated value of $\hat{\tau}_{KL}$ and the value of $\tau$ that generates the corresponding copula, as a function of $\tau$ which is used here as a proxy to the strength of positive monotonic association, for a sample size of $M=1000$. Here, samples of $X = F^{-1}_X(U)$ and $Y=F^{-1}_Y(V)$ are drawn from a Gaussian distribution and from a uniform discrete distribution, respectively, and joined together with four different dependency structures captured by the Gaussian, Frank, Gumbel, and Clayton copulas.  This follows the methodology described by \citet{madsenbirkes} for simulating dependent discrete data.  

Next, we characterize the empirical null distribution of the $\hat{\tau}_{KL}$ estimator.  The linear relationship between the quantiles of the empirical null distribution and a Normal distribution in Fig.~\ref{fig:ktauhat_null}(a) indicate asymptotic normality of the null distribution of $\hat{\tau}_{KL}$, denoted by $X \bigCI Y$.  Fig.~\ref{fig:ktauhat_null}(b) and (c) show that $\hat{\tau}_{KL}$ is Gaussian with a sample mean of approximately zero and a decreasing sample standard deviation as $M$ increases for continuous, discrete, and hybrid random variables.    

We conclude from Table~\ref{tab:step} and Fig.~\ref{fig:ktauhat_bias_properties} that $\hat{\tau}_{KL}$ achieves similar or slightly better bias and variance performance compared to $\hat{\tau}_N$ for hybrid random variables with copula-based dependencies, and far better performance than $\hat{\tau}_b$, while providing the best performance for perfect comonotonic and countermonotonic association patterns.  In the asymptotic case, across the four copulas and marginal distribution types tested, we observe that $\hat{\tau}_{KL}$ and $\hat{\tau}$ both have low bias when the strength of dependency, as referenced by $\tau$, is low ($0 - 0.5$).  When the strength of dependence is medium, $(0.5-0.7)$, $\hat{\tau}_N$ has slightly lower bias.  In the high dependence case, we see that $\hat{\tau}_{KL}$ has lower bias.  These results, combined with the asymptotic normality of the null distribution makes $\hat{\tau}_{KL}$ a viable estimator of the strength of nonlinear monotonic dependence structures, regardless of the type of marginal distribution (discrete, continuous, or hybrid).  



\begin{figure}
\centering
	\begin{subfigure}[b]{0.40\textwidth}
		\includegraphics[width=\linewidth]{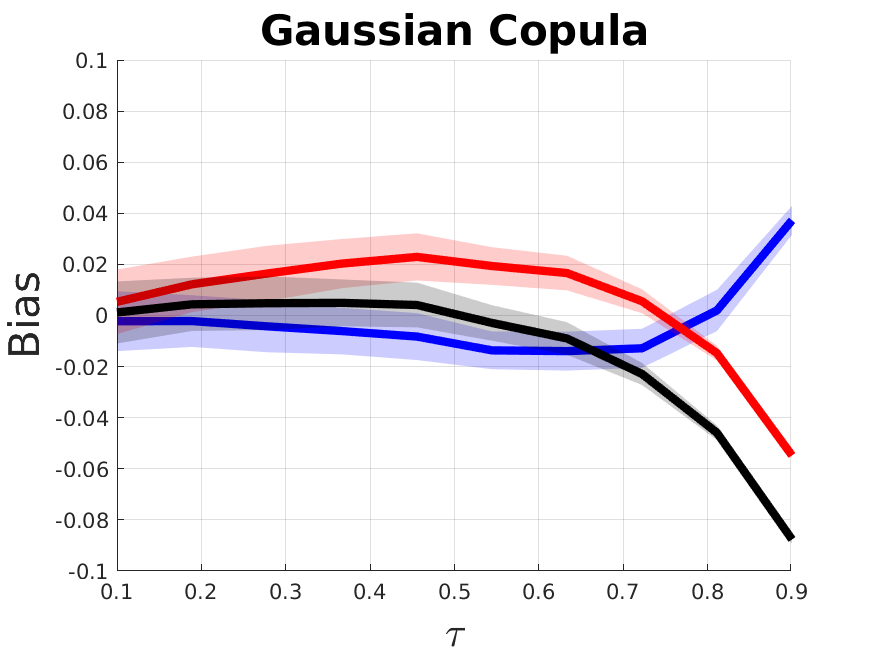}
		\caption{}
	\end{subfigure}%
	\begin{subfigure}[b]{0.40\textwidth}
		\includegraphics[width=\linewidth]{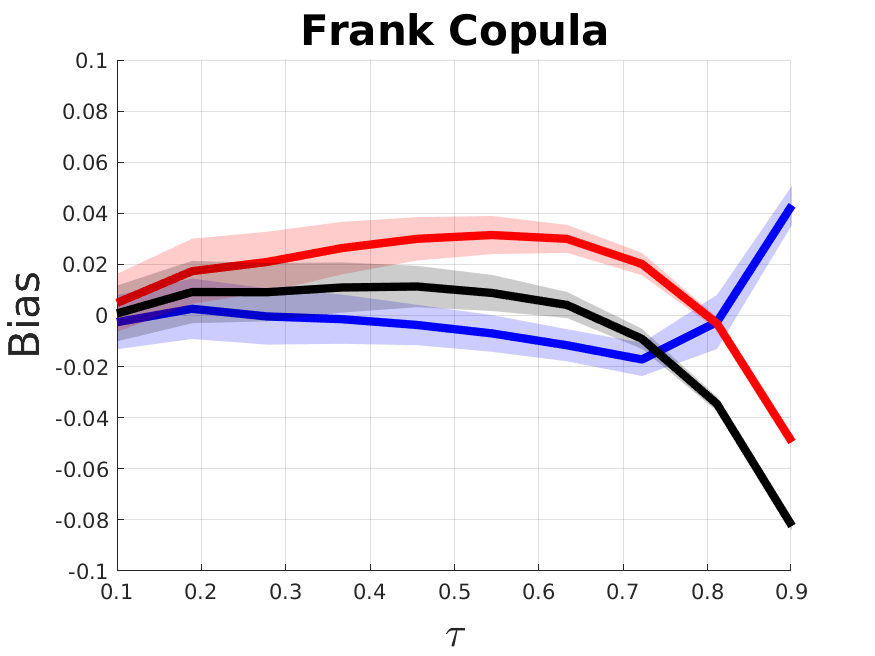}
		\caption{}
	\end{subfigure}\
	\begin{subfigure}[b]{0.40\textwidth}
		\includegraphics[width=\linewidth]{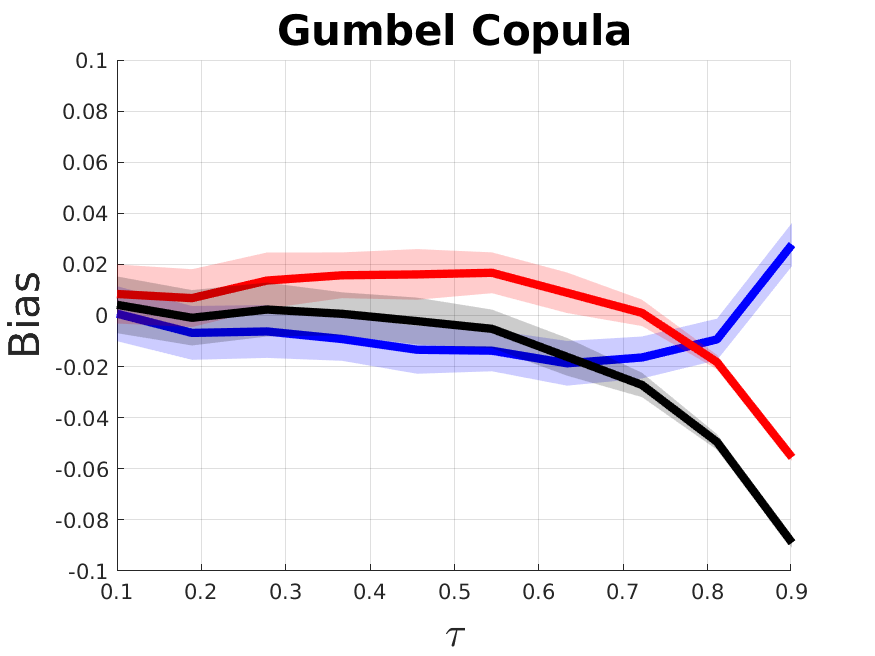}
		\caption{}
	\end{subfigure}%
	\begin{subfigure}[b]{0.40\textwidth}
		\includegraphics[width=\linewidth]{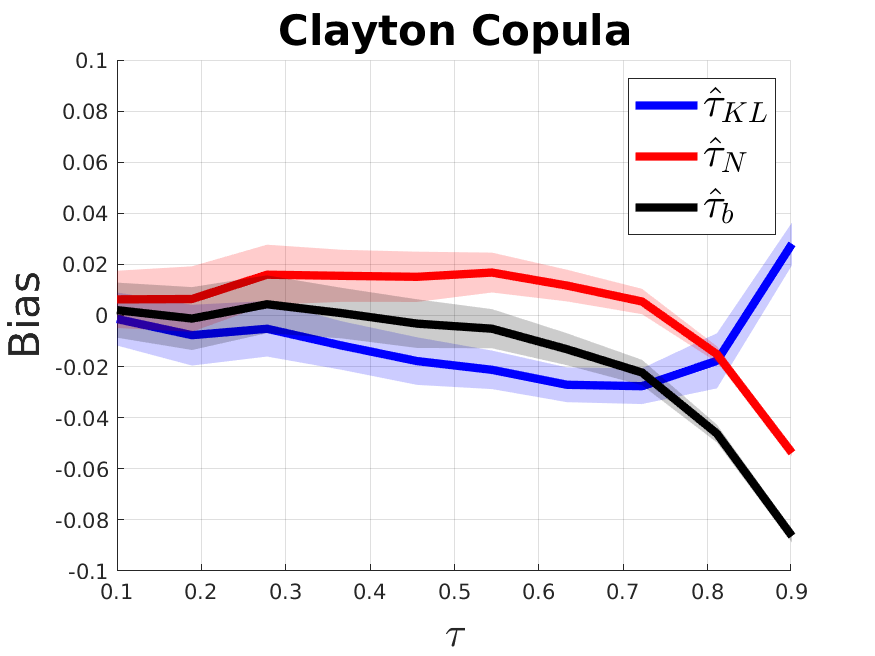}
		\caption{}
	\end{subfigure}\
	\caption{The bias and standard deviation of $\hat{\tau}_b$, $\hat{\tau}_{N}$, and $\hat{\tau}_{KL}$ for varying strengths of dependency for hybrid random variables. The bias and variance for each dependence strength was computed for $M=1000$ for 100 Monte-Carlo simulations.  
	}
	\label{fig:ktauhat_bias_properties}
\end{figure}

\begin{figure}
\centering
    \includegraphics[width=.85\textwidth]{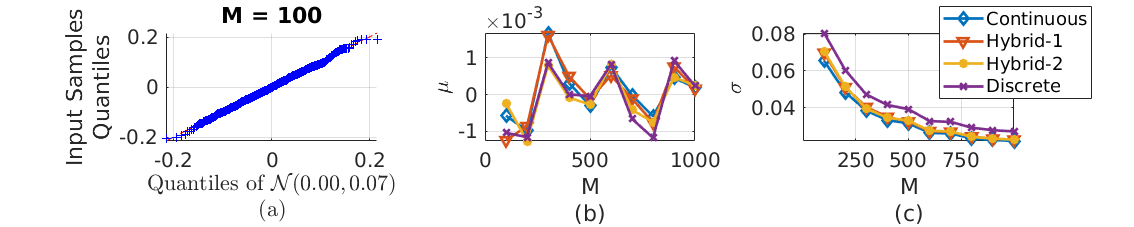}
	\caption{(a) QQ-Plot of $\hat{\tau}_{KL}$ for continuous random variables with $X \bigCI Y$ and $M=100$, (b) The sample mean of the distribution of $\hat{\tau}_{KL}$ for $X \bigCI Y$ as a function of $M$ (sample size), (c) The sample standard deviation of the distribution of $\hat{\tau}_{KL}$ for $X \bigCI Y$ as a function of $M$ (sample size).  Note: Hybrid-1 refers to a discrete X and continuous Y, Hybrid-2 refers to a continuous X and discrete Y.}
	\label{fig:ktauhat_null}
\end{figure}

\section{Copula Index for Detecting Dependence and Monotonicity between Stochastic Signals} \label{sec:cim}
In the previous section, we described an extension to the estimator of $\tau_N$ to account for hybrid random variables.  However, $\tau_N$ is still a rank statistic and thus cannot measure nonmonotonic dependencies.  Here, we describe \textit{CIM}, which is an extension of $\tau_{N}$ to detect nonlinear, nonmonotonic statistical dependencies that satisfies most of R\'enyi's properties and the data processing inequality (DPI).  The motivation for this development comes from the need to assess the strength of association for any general dependence structures that may not be monotonic, when exploring real-world datasets for both analysis and stochastic modeling perspectives, and constructing Markov networks from data, to name a few.  The theoretical foundations of this methodology are first developed.  We then describe the properties of \textit{CIM} and propose an algorithm to estimate it.  

\subsection{Theoretical Foundations of \textit{CIM}}
\textit{CIM} detects statistical dependencies by leveraging concepts from concordance, defined above in (\ref{eq:qfunc}).  However, measures of concordance do not perform well for measuring nonmonotonic dependencies.  This is because two random variables can be perfectly associated, while having the probability of concordance, $P[(X_1-X_2)(Y_1-Y_2)>0]$, equal to the probability of discordance, $P[(X_1-X_2)(Y_1-Y_2)<0]$, yielding a null concordance function $Q$.  An example of such an association is $Y = X^2$, with $X \sim U[-1,1]$.  Thus, in order to use concordance as a measure of nonmonotonic dependence, one must consider regions of concordance and discordance separately; this provides the basis of \textit{CIM}, which computes a weighted average of $|\tau_{N}|$ for each of these regions.

To develop \textit{CIM}, we begin by proving that a set of observations drawn from any mapping can be grouped into concordant and discordant subsets of pseudo-observations that are piecewise linear functions of each other.  Let $F_{X_d}(x_d(m))$ be the $m^{th}$ pseudo-observation for the $d^{th}$ dimensional data point and denote the range-space of $(X,Y)$, where $X$ and $Y$ are random variables, to be the subset of $\mathcal{R}^2$ which encompasses every pair of values that the bivariate random variable $(X,Y)$ can take on.  We can then state the following theorem:

\begin{theorem} \label{thm1}
    Suppose $X$ and $Y$ are random variable associated through a union of functions $g = \bigcup_{i=1}^k h_i$, where $h_i \ \forall i=1 \dots k$ is monotone over the intervals $I_i, i=1 \dots n$, of the real line.  Define the random variables $U = F_X(X)$ and $V=F_Y(Y)$.  Then, $V$ is a piecewise linear function of $U$.  
\end{theorem}
\textbf{Remark:} $g$ is not necessarily a function.  As an example, take $h_1(x) = \sqrt{1-x^2}$ and $h_2(x) = -\sqrt{1-x^2}$.  Here, $h_1$ and $h_2$ are functions, but the union is a circular association pattern which is not a function.
\begin{proof}
    Define $P_i \coloneqq P(X \in I_i)$, $\Omega_i \coloneqq (I_i \times J_i)$ where $J_i$ is the interval of the $y$-axis of the range of $h_i$, with the constraints $\bigcup \Omega_i = \Omega$ and $\Omega_i \cap \Omega_j = \emptyset \ \forall i \neq j$.  Additionally, define $X_i = X \ \textbf{if} \ X \in I_i$ and  $Y_i = h_i(X_i) = Y \ \textbf{if} \ Y \in J_i$.  Finally, define $U_i = F_{X_i}(X_i)$ and $V_i = G_{Y_i}(Y_i)$.  Since we have $\bigcup \Omega_i = \Omega$ and $\Omega_i \cap \Omega_j = \emptyset \ \forall i \neq j$, then $\mathbf{U} = \bigcup_i U_i$ and $\mathbf{V} = \bigcup_i V_i$.

	We can then write
	\begin{align*}
	V_i &= G_{Y_i}(Y_i) \\ 
	    &= P(h_i(X_i) \leq y_i) = \begin{cases}
	F_{X_i}(h_i^{-1}(y_i))    & \text{if} \ \ h_i^{-1}(y) \ \  \text{is increasing} \\
	1-F_{X_i}(h_i^{-1}(y_i))  & \text{if} \ \ h_i^{-1}(y) \ \  \text{is decreasing} \\
	K_i  & \text{if} \ \ h_i^{-1}(y) \ \  \text{is constant}
	\end{cases},
	\end{align*}
	where $K_i = F_{X_{i-1}}(h_{i-1}^{-1}(y_{i-1}))$ if $h_{i-1}^{-1}(y)$ is increasing or $K_i = 1-F_{X_{i-1}}(h_{i-1}^{-1}(y_{i-1}))$ if $h_{i-1}^{-1}(y)$ is decreasing.  Because $F_{X_i}(h_i^{-1}(y_i)) = U_i$, $V_i$ is a peicewise linear function of $U_i$, and thus $V$ is a peicewise linear function of $U$.
\end{proof}

Theorem \ref{thm1} shows that if two random variables are associated in a deterministic sense, their Cumulative Distribution Functions (CDFs) are piecewise linear functions of each other.  This implies that the pseudo-observations of realizations of these dependent random variables can be grouped into regions of concordance and discordance.  Furthermore, in each region, the dependent variable's pseudo-observations are linear functions of the independent ones, contained in the unit square $\mathbf{I}^2$.  Using this as a basis, \textit{CIM} detects dependencies by identifying regions of concordance and discordance after transforming the original data, $x$ and $y$, into the pseudo-observations, $F_X(x)$ and $F_Y(y)$, respectively.

As displayed in Fig.~\ref{fig:cim_indep}, by definition of concordance, in the independence scenario, no regions of concordance or discordance exist.  Similarly, as depicted in Fig.~\ref{fig:cim_mono}, for monotonic dependencies only one region, $\mathbf{I}^2$, exists.  Finally, for nonmonotonic dependencies, many regions may exist.  As an example, Fig.~\ref{fig:cim_sinu} displays the pseudo-observations of sinusoidal functional dependence.  Here, it is easy to see that $R_1$ and $R_3$ are regions of concordance, and $R_2$ is a region of discordance.  

\begin{figure}
	\centering
	\begin{subfigure}[t]{0.30\textwidth}
		\includegraphics[width=1\textwidth]{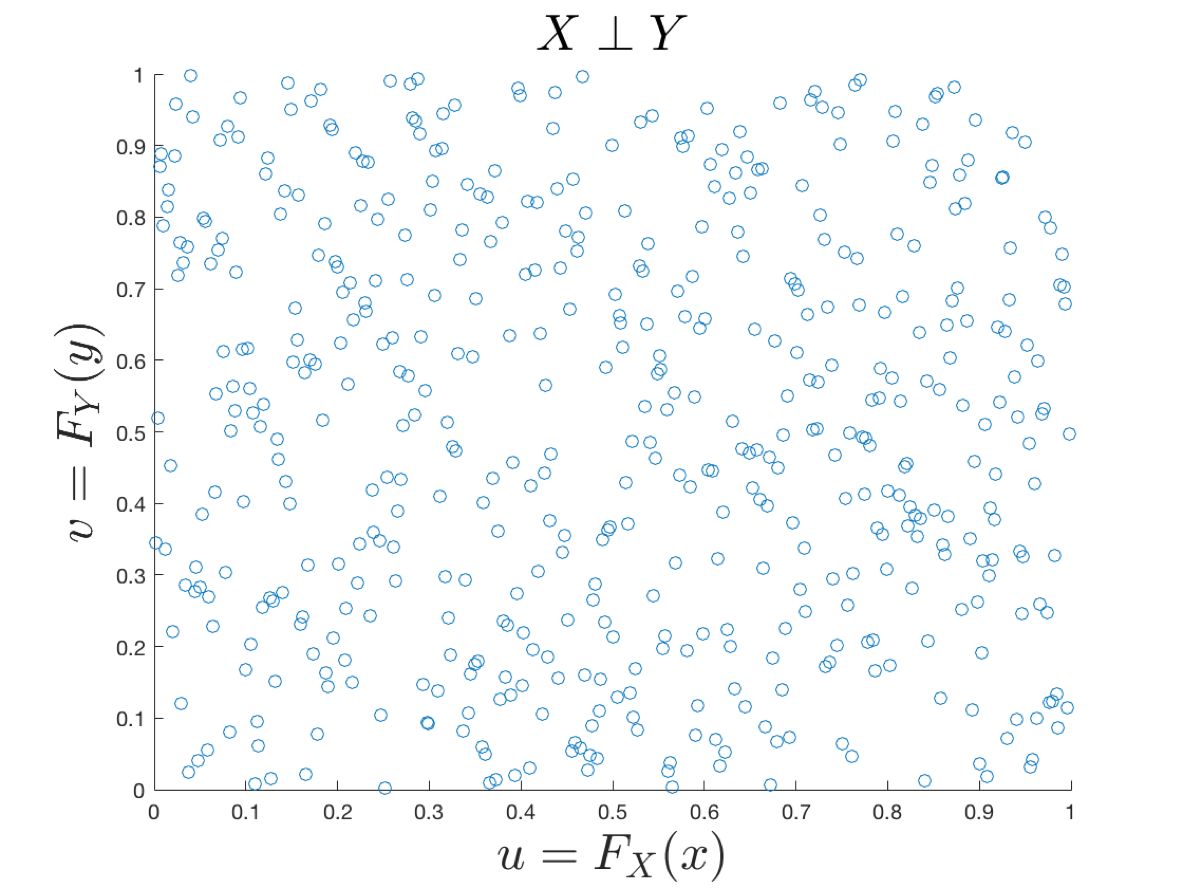}
		\caption{}
		\label{fig:cim_indep}
	\end{subfigure}
	\begin{subfigure}[t]{0.30\textwidth}
		\includegraphics[width=1\textwidth]{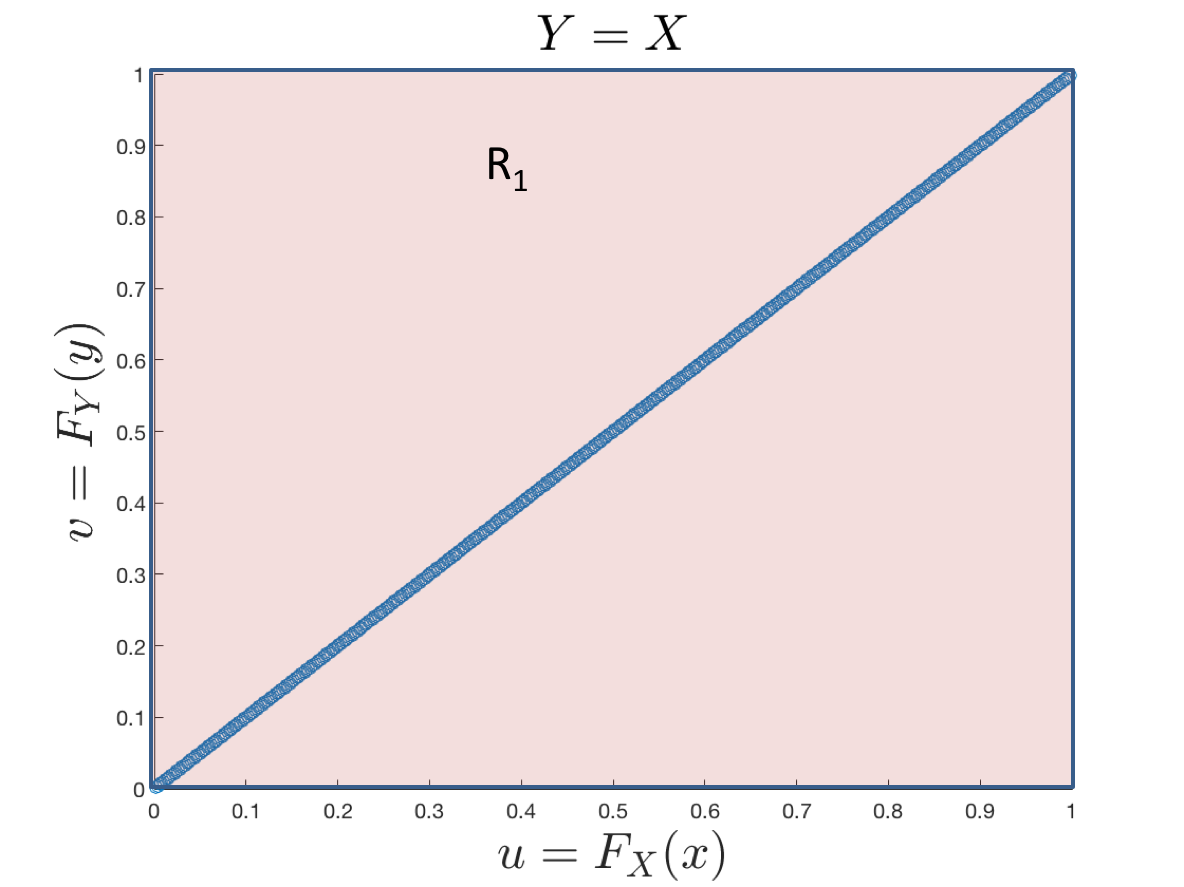}
		\caption{}
		\label{fig:cim_mono}
	\end{subfigure}
	\begin{subfigure}[t]{0.30\textwidth}
		\includegraphics[width=1\textwidth]{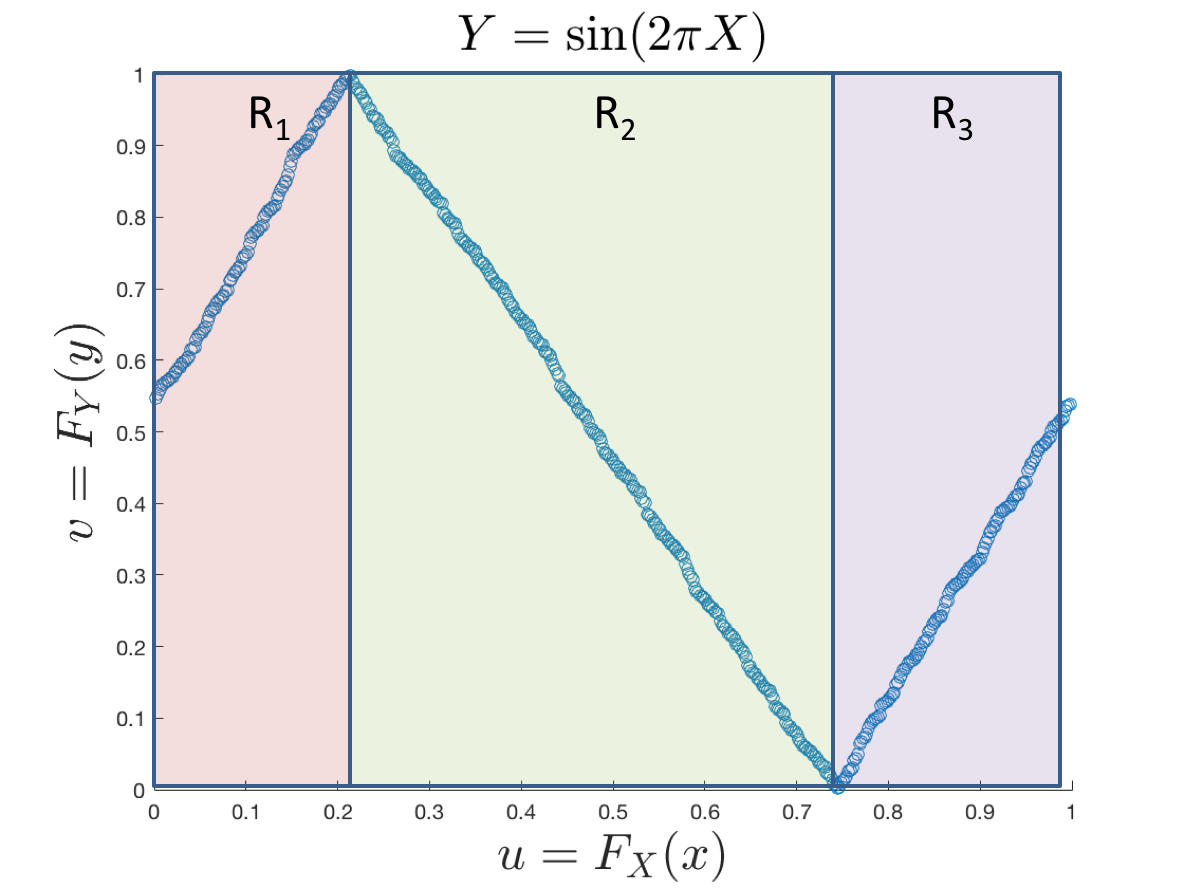}
		\caption{}
		\label{fig:cim_sinu}
	\end{subfigure}
	\caption{Regions of concordance and discordance for three different scenarios: (a) shows two independent random variables, in which case by definition there are no regions of concordance or discordance; (b) shows comonotonic random variables, in which there is one region of concordance, $R_1$; (c) shows a sinusoidal dependence between two random variables, in which there are two regions of concordance, $R_1$ and $R_3$, and one region of discordance, $R_2$.}
	\label{fig:regions}
\end{figure} 
The foregoing examples motivate the following definition of \textit{CIM}: 

\begin{equation}\label{cimeq}
CIM = \sum_i \left( w_i |\tau_{N}^i| \right),
\end{equation}
where $|\tau_{N}^i|$ is the absolute value of $\tau_N$ for the $i^{th}$ region and $w_i$ is the ratio of the area of region $R_i$ to $\mathbf{I}^2$.  From (\ref{cimeq}) and the properties of $\tau_{N}$, \textit{CIM} reduces to $\tau$ for monotonic continuous random variables, and zero for independent random variables.  It should be noted that (\ref{cimeq}) defines \textit{CIM} metric, but an algorithm is required in order to identify each region for which $\tau_{N}$ is computed.  In Section \ref{sec:cim_algo}, we propose an algorithm to identify these regions.  We briefly note that the idea of copula based dependence measures for piecewise monotonic dependence structures has been independently introduced elsewhere in literature~\cite{liebscher2017copula}.

\subsection{Properties of \textit{CIM}}
In this section, we describe the properties of \textit{CIM} defined in (\ref{cimeq}).  We begin by discussing R\'enyi's seven properties of dependence measures, and show that \textit{CIM} satisfies most of them.  We then prove that \textit{CIM} satisfies the Data Processing Inequality (DPI), which implies that it satisfies self-equitability.  Finally, we briefly discuss Reshef's definition of equitability and it's application to \textit{CIM}.

\subsubsection{Dependence Metric Properties}\label{sec:renyi}
\citet{Renyi1959} defined seven desirable properties of a measure of dependence, $\rho^*(X,Y)$, between two random variables $X$ and $Y$:

\begin{enumerate} \label{tbl:renyiproperties}
	\item $\rho^*(X,Y)$ is defined for any pair of non-constant random variables $X$ and $Y$.
	\item $\rho^*(X,Y)$ = $\rho^*(Y,X)$.
	\item $0 \leq \rho^*(X,Y) \leq 1$.
	\item $\rho^*(X,Y) = 0$ iff $X \bigCI Y$.
	\item For bijective Borel-measurable functions, $f$, $g$: $\mathbb{R} \rightarrow \mathbb{R}, \rho^*(X,Y) = \rho^*(f(X),g(Y))$.
	\item $\rho^*(X,Y) = 1$ if for Borel-measurable functions $f$ or $g$, $Y = f(X)$ or $X = g(Y)$.
	\item If $(X,Y) \sim \mathcal{N}(\bm{\mu}, \bm{\Sigma})$, then, $\rho^*(X,Y) = |\rho(X,Y)|$, where $\rho$ is the correlation coefficient.
\end{enumerate}

\begin{theorem}
\textit{CIM} satisfies properties $1-3$, $5$, and $6$ strictly, and can be transformed using a known relation to satisfy property $7$.
\end{theorem}

\begin{proof}
\textit{CIM} satisfies the first property since it operates on copula transformed data (pseudo-observations, which exist for any random variable) rather than the raw data.  Because of the following two identities: $\Sigma_i w_i = 1$, $\mathbf{min}(|\tau_{N}^i|) = 0$ and $\mathbf{max}(|\tau_{N}^i|) = 1$, the value of \textit{CIM} given by (\ref{cimeq}) takes values between 0 and 1, and thus the third property is satisfied.  In the independence case, because there are no regions of concordance or discordance, (\ref{cimeq}) reduces to $|\tau_{N}| = 0$.  From \cite{scarsini}, any measure of concordance is equal to zero when $X$ and $Y$ are independent; because \textit{CIM} reduces to $|\tau_{N}|$, which is an absolute value of the concordance measure $\tau$ for independent random variables, we can state that if $X \bigCI Y$, then $CIM=0$.  The fifth property is also satisfied because Kendall's $\tau$ is invariant to increasing or decreasing transforms \cite[see][Theorem 5.1.8]{nelsen}, so the convex sum of Kendall's $\tau$ must also be invariant to increasing or decreasing transforms.  
The second and sixth properties are satisfied by virtue of Theorem \ref{thm1}.  The seventh property is weakly satisfied because \textit{CIM} metric is the absolute value of Kendall's $\tau$ for a Gaussian copula and can be converted to the correlation coefficient $\theta$, with the relation $\theta = \sin(\frac{CIM \pi}{2})$.  This works because the Gaussian copula captures monotonic linear dependence, and hence there is only one region.

The fourth property is unfortunately not satisfied, due to the copula indifference property.  If $C(u,v) = u - C(u,1-v) = v-C(1-u,v) \ \forall (u,v) \in [0,1]^2$ is satisfied (i.e. the copula is indifferent), $\tau = 0$, and hence $CIM$ cannot be guaranteed to be $0$ only under the independence copula, $\Pi$.

\end{proof}

\subsubsection{Self Equitability and the Data Processing Inequality} \label{sec:dpi}
As noted by \citet{mic_not_equitable}, the DPI and self equitability are important, desirable properties of a dependence metric.  In this section, we prove that $|\tau|$ and \textit{CIM} both satisfy the DPI, and are thus both self-equitable for continuous random variables.  We show that the scaling factors proposed in (\ref{eq:taucj}) and (\ref{eq:taukl}) to account for discrete and continuous random variables, unfortunately, does not satisfy the DPI.  We then propose a solution to allow \textit{CIM} to satisfy the DPI, even in the discrete and hybrid scenarios.

The DPI is a concept that stems from information theory.  It states that if random variables $X$, $Y$, and $Z$ form a Markov chain, denoted by $X \rightarrow Y \rightarrow Z$, then $I(X;Y) > I(X;Z)$, where $I(X;Y)$ is the mutual information between $X$ and $Y$ defined as $$I(X;Y) = \int_Y \int_X f_{XY}(x,y) \mathbf{log} \frac{f_{XY}(x,y)}{f_X(x)f_Y(y)} dx dy,$$ where $f_{XY}(x,y)$ is the joint distribution of $X$ and $Y$, and $f_X(x)$ and $f_Y(y)$ are the marginal densities of $X$ and $Y$, respectively \cite{CoverInfoTheo}.  Intuitively, it asserts that information is never gained when being transmitted through a noisy channel \cite{mic_not_equitable}.  As an analog to the information theoretic definition of the DPI, \citet{mic_not_equitable} define a dependence metric $D$ to satisfy the DPI if and only if $D(X;Y) \geq D(X;Z)$, whenever the random variables $X$, $Y$, and $Z$ form the Markov chain, $X \rightarrow Y \rightarrow Z$.  Here, we prove that \textit{CIM}, as defined by (\ref{cimeq}), satisfies the DPI. 

\begin{theorem} \label{thm:cim_dpi}
	If the continuous random variables $X$, $Y$, and $Z$ form a Markov chain $X \rightarrow Y \rightarrow Z$, then $CIM(X,Y) \geq CIM(X,Z)$. 
\end{theorem}

\begin{proof}
    Our approach is to show that $|\tau|$ satisfies DPI, and then utilize that result to show that $CIM$ satisfies DPI.
    We can rewrite the copula of $C_{YZ}$ as the sum of the convex combinations of all patches over the unit-square as
	\begin{align*}
	C_{YZ}(u,v) &= \sum_{i=0}^{m-1} \sum_{j=0}^{m-1} p_{ij} \left[ \alpha_{ij} M^{ij}(u,v) + \beta_{ij} \Pi^{ij}(u,v) + \gamma_{ij} W^{ij}(u,v) \right] \\
	\implies \frac{\partial C_{YZ}(t,v)}{\partial t} &= \sum_{i=0}^{m-1} \sum_{j=0}^{m-1} p_{ij} \left[ \alpha_{ij} \frac{\partial M^{ij}(t,v) }{\partial t} + \beta_{ij} \frac{\partial \Pi^{ij}(t,v) }{\partial t} + \gamma_{ij} \frac{\partial W^{ij}(t,v) }{\partial t} \right]
	\end{align*}
	where $\alpha_{ij} + \beta_{ij} + \gamma_{ij} = 1$ and $\sum_{i=0}^{m-1} \sum_{j=0}^{m-1} p_{ij} = 1$.  Substituting and utilizing the relation $C_{XZ}(u,v) = C_{XY}*C_{YZ}(u,v) = \int_0^1 \frac{\partial C_{XY}(u,t)}{\partial t} \frac{\partial C_{YZ}(t,v)}{\partial t} dt$~\cite{darsow1992}, we get
	
	\begin{align*}
	C_{XZ}(u,v) &= \sum_{i=0}^{m-1} \sum_{j=0}^{m-1} \bigg[ \alpha_{ij} \int_0^1 \frac{\partial C_{XY}^{ij}(u,t)}{\partial t} \frac{\partial M^{ij}(t,v)}{\partial t} dt + \beta_{ij} \int_0^1 \frac{\partial C_{XY}^{ij}(u,t)}{\partial t} \frac{\partial \Pi^{ij}(t,v)}{\partial t} dt \\ 
	& + \gamma_{ij} \int_0^1 \frac{\partial C_{XY}^{ij}(u,t)}{\partial t} \frac{\partial W(t,v)}{\partial t} dt \bigg] \\ 
	& = \sum_{i=0}^{m-1} \sum_{j=0}^{m-1} \left[ \alpha_{ij} [C_{XY}^{ij}*M^{ij}] + \beta_{ij} [C_{XY}^{ij}*\Pi^{ij}] + \gamma_{ij} [C_{XY}^{ij}*W^{ij}] \right]
	\end{align*}
	For each patch, omitting the indices $ij$ and substituting the relations $C*M=C, C*\Pi=\Pi$, and $C*W(u,v)=u-C(u,1-v)$~\cite{darsow1992}, we get
	
	\begin{equation*}
	C_{XZ} = \alpha C_{XY} + \beta \Pi + \gamma [u-C_{XY}(u,1-v)]
	\end{equation*}
	Due to the 2-increasing property of copulas and $\alpha + \beta + \gamma = 1$, we can say that for each patch, $C_{XY} \succeq C_{XY} \implies C_{XY} \succeq \alpha C_{XY}$, and $C_{XY} \succeq \Pi \implies C_{XY} \succeq \beta \Pi$.  Additionally, by assumption, we have 
	\begin{align*}
	C_{XY} \succeq \Pi & \implies C_{XY}(u,v) \geq \Pi(u,v) \ \forall u,v \in \mathbf{I} \\
	& \implies C_{XY}(u,1-v) \geq \Pi(u,1-v) \\
	& \implies C_{XY}(u,1-v) \geq u-\Pi(u,v) \\
	& \implies C_{XY}(u,1-v) \geq u + (-C_{XY}(u,v)) \tag{\textit{because} $C_{XY}(u,v) \geq \Pi(u,v) \  implies -C_{XY}(u,v) \leq -\Pi(u,v)$} \\
	& \implies C_{XY}(u,v) \geq u - C_{XY}(u,1-v) \\
	& \implies C_{XY} \succeq C_{XY}*W \\
	& \implies C_{XY} \succeq \gamma(C_{XY}*W)
	\end{align*}
	Thus, $C_{XY} \succeq C_{XZ}$ with the constraint that $C_{XZ} \succeq \Pi$ for every patch, and because we have a convex combination of patches and increasing the number of patches, $m$, decreases the approximation error to an arbitrarily small amount, it follows that $|\tau|$ satisfies DPI.
	
	
	Now, define $X_i = \begin{cases} X \ \textbf{if} \ X \in \mathcal{D}(i) \\ 0 \ \textbf{else} \end{cases}$, where $\mathcal{D}(i)$ denotes the domain of the $i^{th}$ patch, such that $P(X_i \in \mathcal{D}(i), X_j \in \mathcal{D}(j)) = 0 \ \forall \ i \neq j$ and $\sum_i P(X_i \in \mathcal{D}(i)) = 1$, and $Y_i$ to be the corresponding range of patch $i$, for the joint distribution $f_{XY}$.  Similarly define $Z_i$ to be the range of the $i^{th}$ patch for the joint density $f_{XZ}$.  Then,
	
	\begin{align*}
	f_{X_i,Y_i}(x_i,y_i) & = \frac{f_{XY}(x,y)}{f_{X_i|X}(x_i|x)} \\
	f_{X_i,Z_i}(x_i,z_i) & = \frac{f_{XZ}(x,z)}{f_{X_i|X}(x_i|x)}
	\end{align*}
	
	Recall that because $X$, $Y$, and $Z$ satisfy DPI, the relation
	
	\begin{equation*}
	\int_Y \int_X f_{XY}(x,y) \textbf{log} \left( \frac{f_{XY}(x,y)}{f_X(x) f_Y(y)} \right) dx dy \geq \int_Z \int_X f_{XZ}(x,z) \textbf{log} \left( \frac{f_{XZ}(x,z)}{f_X(x) f_Z(z)} \right) dx dy
	\end{equation*}
	holds.  Additionally, $f_{X_i|X}(x_i|x)$ is a constant.  Hence, the following must hold:
	
	\begin{align*}
	\int_Y \int_X \frac{f_{XY}(x,y)}{f_{X_i|X}(x_i|x)} & \textbf{log} \left( \frac{f_{XY}(x,y)}{f_{X_i|X}(x_i|x) f_X(x) f_Y(y)} \right) dx dy \geq \\
	&  \int_Z \int_X \frac{f_{XZ}(x,z)}{f_{X_i|X}(x_i|x)} \textbf{log} \left( \frac{f_{XZ}(x,z)}{f_{X_i|X}(x_i|x) f_X(x) f_Z(z)} \right) dx dz \\
	\implies \int_{Y_i} \int_{X_i} f_{X_i Y_i}(x_i,y_i) \textbf{log} & \left( \frac{f_{X_i Y_i}(x_i,y_i)}{f_{X_i}(x_i) f_{Y_i}(y_i)} \right) dx_i dy_i \geq \\
	& \int_{Z_i} \int_{X_i} f_{X_i Z_i}(x_i,z_i) \textbf{log} \left( \frac{f_{X_i Z_i}(x_i,z_i)}{f_{X_i}(x_i) f_{Z_i}(z_i)} \right) dx_i dz_i \\
	& \implies X_i \rightarrow Y_i \rightarrow Z_i \\
	& \implies C_{X_i, Y_i} \succeq C_{X_i, Z_i} \\
	& \implies |\tau(X_i, Y_i)| \geq |\tau(X_i, Z_i)|
	\end{align*}
	
	Because $X_i$, $Y_i$, and $Z_i$ is shown to satisfy the DPI, 
	
	\begin{align*}
	\sum_i w_i |\tau(X_i, Y_i)| \geq \sum_i w_i |\tau(X_i, Z_i)| \\
	\implies CIM(X,Y) \geq CIM(X,Z)
	\end{align*}
	because $\sum_i w_i = 1$.  
\end{proof}  

An immediate implication of \textit{CIM} satisfying DPI is that it is a self-equitable statistic.  A dependence measure $D(X;Y)$ is said to be self-equitable if and only if it is symmetric, that is,  ($D(X;Y)=D(Y;X)$), and satisfies $D(X;Y)=D(f(X);Y)$, whenever $f$ is a deterministic function, $X$ and $Y$ are variables of any type, and $X \rightarrow f(X) \rightarrow Y$, implying that they form a Markov chain \cite{mic_not_equitable}.  Self equitability implies that $CIM(X,Y)$ is invariant under arbitrary invertible transformations of $X$ or $Y$ \cite{mic_not_equitable}, which is in-fact a stronger condition than R\'enyi's $5^{th}$ property given in Section 3.2.1.  Because $\tau_N$ also satisfies the concordance properties, the DPI proof holds for discrete and hybrid random variables~\cite{neslehova}.

\subsubsection{Equitability and Noise Properties}
Equitability is a measure of performance of a statistic under noise.  Notionally, an equitable statistic assigns similar scores to equally noisy relationships of different types \cite{mic}.  \citet{mic_not_equitable} formalize this concept as $R^2$-equitability.  Recall that a dependence measure $D[X;Y]$ is $R^2$ equitable if and only if, when evaluated on a joint probability distribution $p(X,Y)$, that corresponds to a noisy functional relationship between two real random variables $X$ and $Y$, the relation given by

\begin{equation} \label{eq:r2equitability}
D[X;Y] = g(R^2([f(X); Y]))
\end{equation}
holds true, where $g$ is a function that does not depend on $p(X,Y)$, $R^2$ denotes the squared Pearson correlation measure, and $f$ is the function defining the noisy functional relationship, namely $Y = f(X) + \eta$, for some random variable $\eta$.  

Through simulations, we observe that $\tau$ is not an equitable metric.  Following \citet{reshef2015equitability}, we compute the equitability curves, which show the relationship between $\tau$ and $R^2$ for different relationships, for the two association patterns $Y=X$ and $Y=e^X$.  These are displayed in Fig.~\ref{fig:equitability_curves}.  The worst interpretable interval, which can be informally defined as the range of $R^2$ values corresponding to any one value of the statistic is represented by the red hashed line.  Fig.~\ref{fig:equitability_curves} depicts a large interval, which is indicative of the lack of $R^2$-equitability of this estimator.  From a theoretical perspective, this can be understood from (\ref{eq:kendalls_tau_estimator}), which shows that the distances between points are not considered, only their relative rankings. Because $\tau$ is not equitable and \textit{CIM} is based on $\tau$, the latter is also not equitable according to (\ref{eq:r2equitability}).  Additionally, the distance argument leads to the conclusion that any concordance-based measures are not $R^2$-equitable.

\begin{figure}

    \centering
	\begin{minipage}[c]{0.3\textwidth}
        \includegraphics[width=\textwidth]{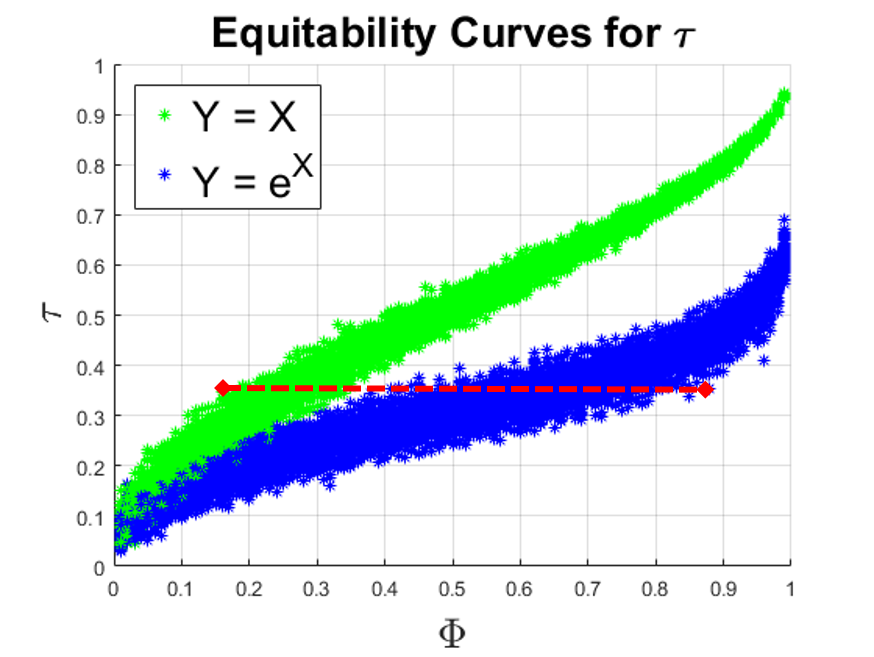}
    \end{minipage}\hfill
   \begin{minipage}[c]{0.6\textwidth}
        \caption{
           Equitability curves for Kendall's $\tau$ for two functional dependencies, where $X \sim U[2,10]$ and $Y=X$ in green and $Y=e^X$ in blue.  Here, we see that the worst interpretable interval, shown by the red hashed line, is large, indicating lack of equitability of $\hat{\tau}$.
        } 
        \label{fig:equitability_curves}
      \end{minipage}
\end{figure}

\subsection{Proposed Algorithms} \label{sec:cim_algo}
In this section, we propose an algorithm to estimate \textit{CIM} metric.  We begin by precisely defining the estimator in a mathematical framework that captures the optimization procedure required to find the regions of monotonicity.  Given an ordered pair of data, $D_i \in D$ such that $D \in \mathcal{R}^2$, let us define $\hat{\tau}_{KL}(D_i)$ to be (\ref{eq:taukl}) for data subset $D_i$, $I_i$ to be the interval over the real line corresponding to the domain of $D_i$, and $J_i$ to be the corresponding range of $D_i$.  For convenience, label this ordered pair to be a region $R_i \coloneqq (D_i \times J_i)$  The domain and range of $D$ can be divided into regions such that


\begin{equation}
\begin{aligned}\label{eq:cim_optim_criterion}
& \underset{(I_i \times J_i)}{\text{max}} |\hat{\tau}_{KL}(D_i)| \ && \forall i =1\dots R, \\
& \textbf{subject to } && (I_i \times J_i) \cap (I_j \times J_j) = \emptyset \\
&&& \bigcup (I_i \times J_i) = D \times R,
\end{aligned}
\end{equation}
where $w_i$ is the ratio of the number of samples in $D_i$ to the total number of samples being considered, and $R$ is the total number of regions.  Then, we can define

\begin{equation}\label{eq:cim_estimator}
    \widehat{CIM}  = \sum_i w_i |\hat{\tau}_{KL}(D_i)|.
\end{equation}

The steps in Listing 1 below present a high level outline of estimating \textit{CIM} index.

\begin{figure}
	\centering
	\begin{subfigure}[t]{0.45\textwidth}
		\includegraphics[width=1\textwidth]{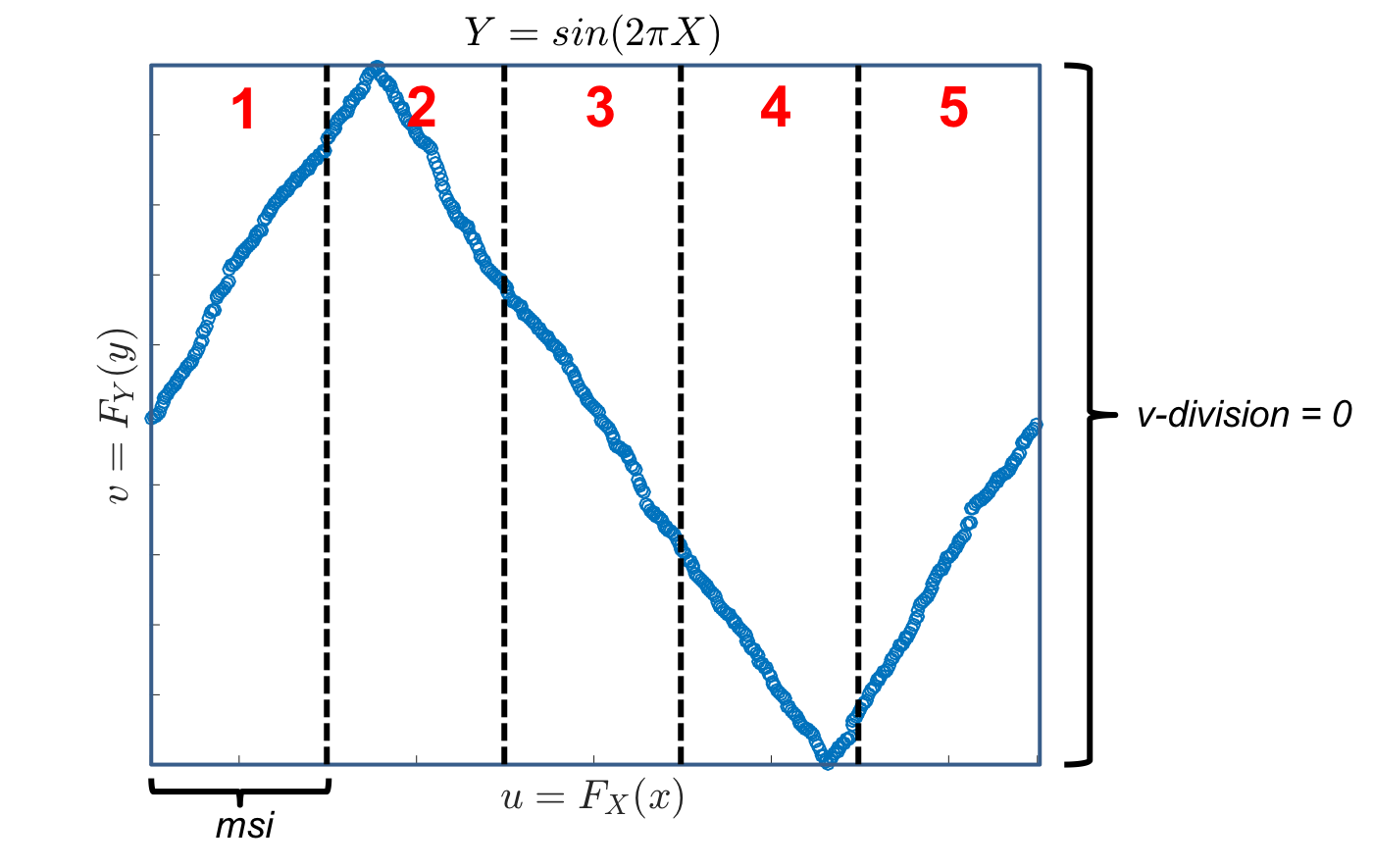}
		\caption{}
		\label{fig:cim_algo_explain_v2_1}
	\end{subfigure}
	\begin{subfigure}[t]{0.45\textwidth}
		\includegraphics[width=1\textwidth]{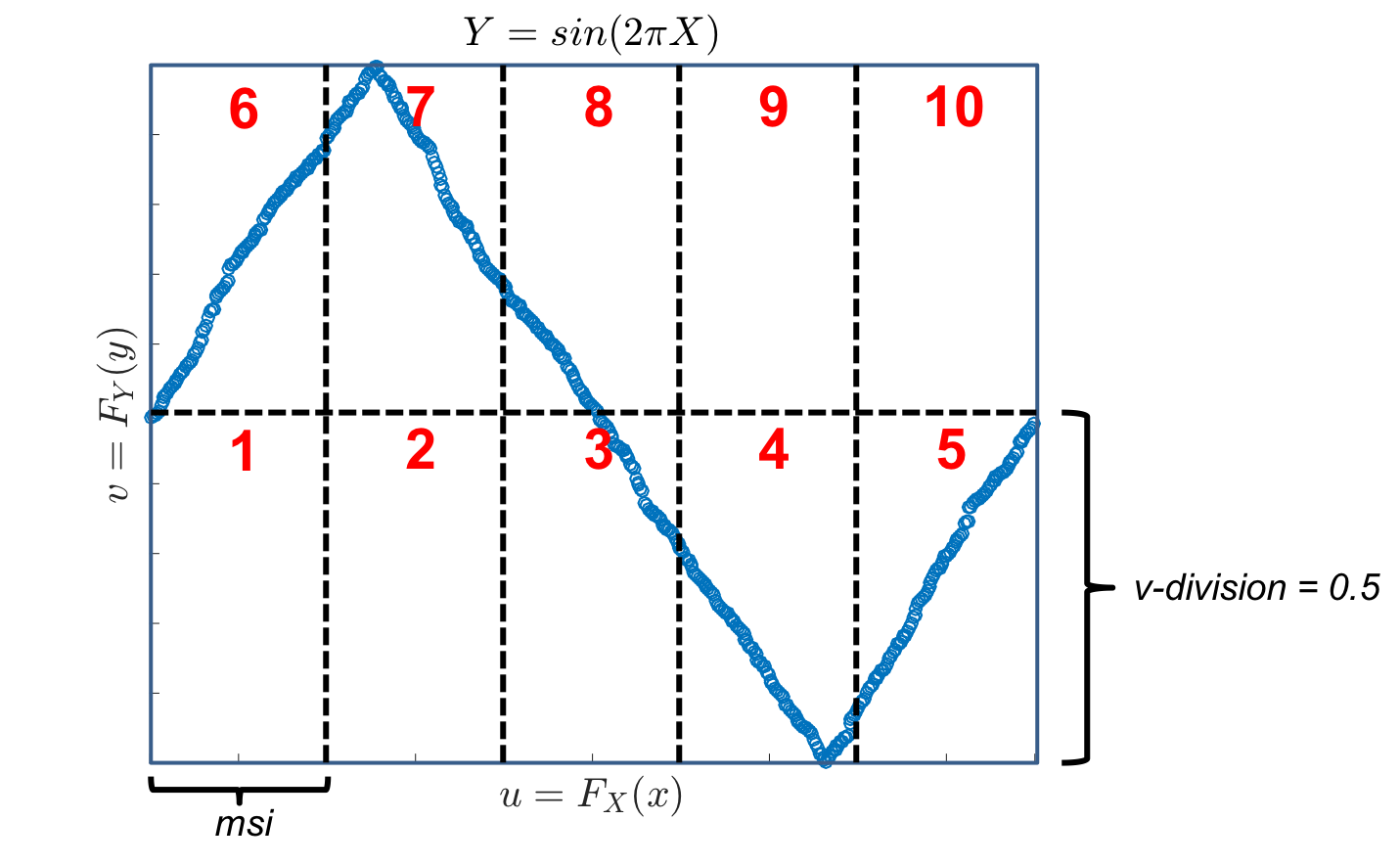}
		\caption{}
		\label{fig:cim_algo_explain_v2_2}
	\end{subfigure}
	\caption{(a) Partition of the unit-square into grids of size $msi$, with $v\text{-division}=0$.  (b) Partition of the unit-square into grids of size $msi$, with $v\text{-division}=0.5$.}
	\label{fig:cim_algo_explain_v2}
\end{figure} 

\begin{mdframed}
\textbf{Listing 1: \textit{CIM} Estimation Overview}
\begin{enumerate}
    \item Transform data into pseudo-observations.
    \item Given an $msi$ and $v$-division, divide the unit-square into grids and number in the pattern shown in Fig.~\ref{fig:cim_algo_explain_v2}.
    
    \item 
    \begin{algorithmic}[1]
        \ForEach {grid $i = 1 \dots n$}
        \State $^i\hat{\tau}_{KL} \gets \tau_{KL}(D_i)$
        \State $^i\hat{\tau}_{KL}' \gets \tau_{KL}(D_i \bigcup D_{i+1})$
        \If{$|^i\hat{\tau}_{KL}| < |^i\hat{\tau}'_{KL}| - \frac{\sigma_{\hat{\tau}_{KL}}}{\sqrt{M_i}} u_{1-\frac{\alpha}{2}}$}
            \State Declare boundary between $R_i$ and $R_{i+1}$ to be a region boundary\footnote{Recommendations on values for $msi$ and $v$-division are provided later in the manuscript}.
        \Else
            \State Merge $R_i$ and $R_{i+1}$ into $R_i$.
        \EndIf
        \EndFor
	\end{algorithmic}
	
    \item Repeat Steps 2 and 3 for different values of $msi$ and $v$-divisions.
    \item Take $\widehat{CIM}$ to be the maximum over all computed combinations in Step 5, as computed by (\ref{eq:cim_estimator}).
\end{enumerate}
\end{mdframed}

More specifically, the first step in approximating \textit{CIM} statistic is to transform the data by applying the probability integral transform, via the empirical cumulative distribution function, to both dimensions of the data independently, generating the pseudo-observations.  Next, the unit square is divided and scanned to identify regions of concordance and discordance, as stated in Listing 2.  The output of this step for independent, linear, and sinusoidal association patterns is shown in Figs.~\ref{fig:regions} (a), (b), and (c), respectively.  
The decision criteria for deciding region boundaries is given by 
\begin{equation}\label{eq:decision_criterion}
|\hat{\tau}_{KL}| < |\hat{\tau}'_{KL}| - \frac{\sigma_{\hat{\tau}_{KL}}}{\sqrt{M}} u_{1-\frac{\alpha}{2}},
\end{equation}
where $\sigma_{\hat{\tau}_{KL}}$ is the standard deviation of the $\hat{\tau}_{KL}$, $M$ is the number of samples, and $u_{1-\frac{\alpha}{2}}$ is the quantile of the standard normal distribution at a significance level of $\alpha$, then a new region boundary is declared.  Stated differently, if the current value $|\hat{\tau}_{KL}|$ has decreased by more than a defined amount of confidence, $\alpha$, from its previously estimated value $|\hat{\tau}'_{KL}|$, then the algorithm declares this a new boundary between monotonic regions.  
Fig.~\ref{fig:cim_algo_explain} pictorially depicts these steps.  In Fig.~\ref{fig:cim_algo_explain} (a), $R_1$ that has been identified by the algorithm as the one that contains points of concordance, noted by $\hat{\tau}'_{KL}$, after several iterations of the for loop in Listing 2.  Additionally, the green region in Fig.~\ref{fig:cim_algo_explain} (a) shows the region under consideration by the algorithm, which is an increment of the one identified by $si$.  $\hat{\tau}'_{KL}$ and $\hat{\tau}_{KL}$ are compared according to the criterion given above.  In Fig.~\ref{fig:cim_algo_explain} (a), the criterion in (\ref{eq:decision_criterion}) yields the decision that the points in the green region belong to the same region, denoted by $R_1$.  In Fig.~\ref{fig:cim_algo_explain} (b), the same criterion in (\ref{eq:decision_criterion}) yields the decision that the points in green belong to a new region, $R_6$, as depicted in Fig.~\ref{fig:cim_sinu}.

In order to maximize the power of \textit{CIM} estimator against the null hypothesis that $X \bigCI Y$, the scanning process is conducted for multiple values of $si$, both orientations of the unit-square ($\varA{u-v}$, and $\varA{v-u}$), and sub-intervals of $u$ and $v$ separately.  The scanning and orientation of the unit square, which maximizes the dependence metric, is the approximate value of \textit{CIM}.  The minimum scanning increment (width of the green region in Fig.~\ref{fig:cim_algo_explain} (a)), noted as $msi$, and the confidence level, $\alpha$, are the two hyperparameters for the proposed algorithm.  The value of $msi$ used in all the simulations, except the sensitivity study, is $\frac{1}{64}$.  The denominator of this value bounds the size and frequency of changes to the monotonicity that the algorithm can detect.  By choosing $\frac{1}{64}$, it is found that all reasonable dependence structures can be captured and identified. The value of $\alpha$ used in all the simulations is $0.2$, which was found to be a good tradeoff between overfitting and detecting new regions experimentally from a statistical power perspective.  The experiments conducted in Section 4.1 and \ref{sec:cim_real_data} corroborate these choices.  The complete pseudocode for estimating \textit{CIM} index is shown in Algorithm \ref{alg:cim} in \ref{appendix_cim_algo}; additionally, a reference implementation is also provided \footnote{\url{https://github.com/stochasticresearch/depmeas/blob/master/algorithms/cim.m}}.

\begin{figure}
	\centering
	\begin{subfigure}[t]{0.35\textwidth}
		\centering
		\includegraphics[width=1\linewidth]{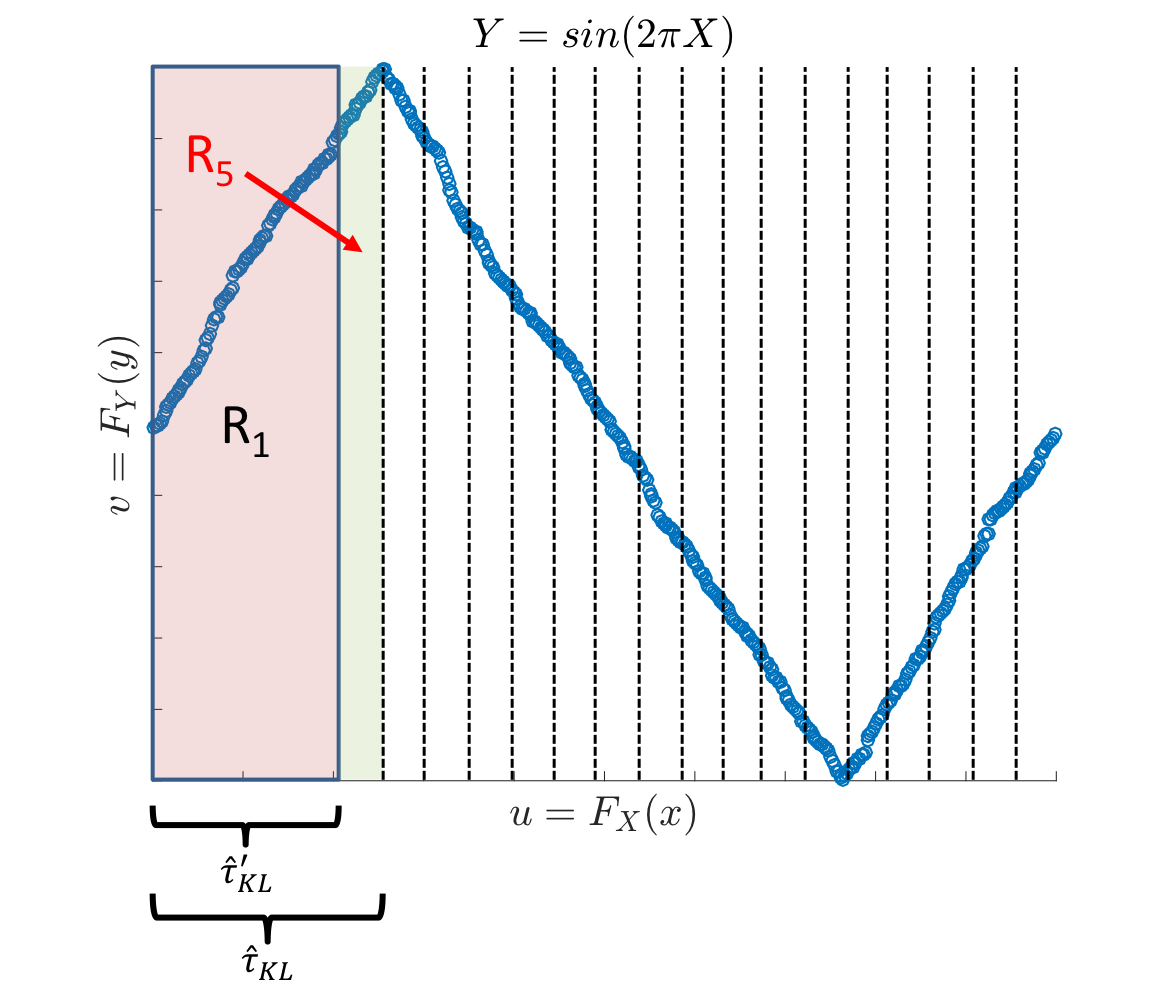}
		\caption{}
		\label{fig:cim_algo_explain1}
	\end{subfigure}%
	\begin{subfigure}[t]{0.35\textwidth}
		\centering
		\includegraphics[width=1\linewidth]{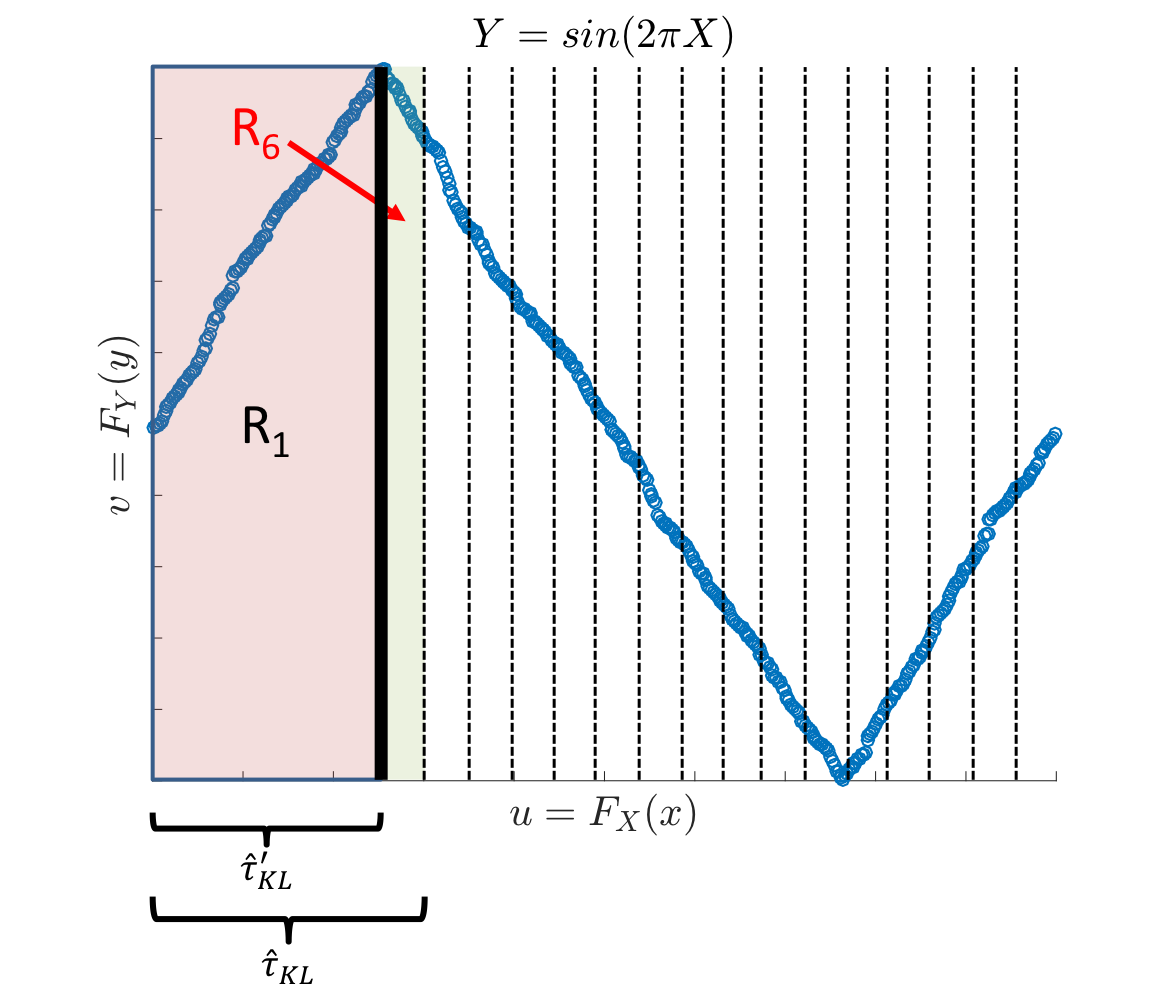}
		\caption{}
		\label{fig:cim_algo_explain2}
	\end{subfigure}
	\caption{Operation of \textit{CIM} algorithm.  In (a), \textit{CIM} algorithm decides that the green region, $D_5$, belongs to the same region as $R_1$, and $R_5$ becomes merged into $R_1$, per the algorithm procedure in Listing 2.  In (b), \textit{CIM} algorithm decides that green region, $R_6$ belongs to a new region, different from $R_1$, due to the decision criterion in (\ref{eq:decision_criterion}) and declares a region boundary, depicted by the solid black line.  In both of these figures, the gridding pattern is shown with the thin dotted black line.}
	\label{fig:cim_algo_explain}
\end{figure}

\subsubsection{Algorithm Validity}\label{sec:alg_validity}
In this section, we discuss the theoretical validity of the convergence of Algorithm \ref{alg:cim}.  From \cite{schmock}, we have

\begin{equation}\label{eq:tau_convergence}
    P\left[ \hat{\tau}_{KL} - \frac{\sigma_{\hat{\tau}_{KL}}}{\sqrt{M}} u_{1-\frac{\alpha}{2}} \leq \tau_{KL} \leq \hat{\tau}_{KL} + \frac{\sigma_{\hat{\tau}_{KL}}}{\sqrt{M}} u_{1-\frac{\alpha}{2}} \right] \xrightarrow{ M \to \infty } 1-\alpha,
\end{equation}
where $M$ is the number of samples available to estimate $\tau_{KL}$ and the other variables were defined above in (\ref{eq:decision_criterion}).  From (\ref{eq:tau_convergence}), we can state the following:

\begin{theorem}\label{thm:tau_convergence}
    The region detection criterion, $|\hat{\tau}_{KL}| < |\hat{\tau}'_{KL}| - \frac{\sigma_{\hat{\tau}_{KL}}}{\sqrt{M}} u_{1-\frac{\alpha}{2}}$, guarantees that as $M \to \infty$, a change in monotonicity in the dependence structure will be detected with a probability of $1-\alpha$, where $\alpha$ is a configurable confidence level, and $M$ is the number of samples available to estimate $\tau_{KL}$.
\end{theorem}

\begin{proof}
With $\alpha \to 0, n \to \infty$, from (\ref{eq:tau_convergence}), \textit{CIM} detection criterion given by (\ref{eq:decision_criterion}) reduces to

\begin{equation*}
    \lvert \tau_{KL} \rvert < \lvert \tau'_{KL} \rvert.
\end{equation*}

Under the assumption that the noise distribution is stationary over the data being analyzed, in the limit as $n \to \infty$, if points belong to the same region, then $\lvert \tau_{KL} \rvert \geq \lvert \tau'_{KL} \rvert$, and $\lvert \tau_{KL} \rvert < \lvert \tau'_{KL} \rvert$ if newly added points belong to a different region.  Thus, as $n \to \infty$, the region detection criterion given by (\ref{eq:decision_criterion}) will detect any region boundary with probability of $1$.

\end{proof}

Theorem \ref{thm:tau_convergence} guarantees that if the unit square is scanned across $v$ for the full-range of $u$, any injective or surjective association pattern's changes in monotonicity will be detected with probability of $1-\alpha$ as $n \to \infty$.  For association patterns which map multiple values of $y$ to one value of $x$ (such as the circular pattern), the range of $u$ is divided and scanned separately.  Because the dependence structure is not known a-priori, various scans of the unit-square are performed at different ranges of $u$ and $v$.  As stated above, the configuration that maximizes the dependence metric is then chosen amongst all the tested configurations.  

\subsubsection{Algorithm Performance}\label{sec:alg_performance}
In this section we investigate the performance of Algorithm \ref{alg:cim} using various synthetic datasets.  We show that the proposed algorithm is robust to both input hyperparameters, $msi$ and $\alpha$.  We also investigate the convergence properties and speed of convergence of $\widehat{CIM}$ as estimated by Algorithm \ref{alg:cim}.  Because the algorithm performance depends heavily on how well it detects the regions of concordance and discordance, we begin by characterizing the region detection performance.

To test the region detection performance, we simulate noisy nonmonotonic relationships of the form
\begin{equation}\label{eq:region_detection_generating}
Y = 4(X-r)^2 + \mathcal{N}(0,\sigma^2),
\end{equation}
where $X \sim U(0,1)$.  By varying $r$ and the number of samples, $M$, that are drawn from $X$, nonmonotonic relationships of this form comprehensively test the algorithm's ability to detect regions for all types of association.  This is because $r$ directly modulates the angle between the two piecewise linear functions at a region boundary, and the number of samples and the noise level test the performance of the decision criterion specified previously in (\ref{eq:decision_criterion}) as a function of the number of samples.  After generating data according to (\ref{eq:region_detection_generating}) for various values of $r$, $M$, and $\sigma$, Algorithm \ref{alg:cim} is run on the data and the boundary of the detected region is recorded, for 500 Monte-Carlo simulations.  A nonparametric distribution of the detected regions by Algorithm \ref{alg:cim} for different values of $r$ and $M$ is displayed in Fig.~\ref{fig:region_detection_performance}.  It is seen that on average, the algorithm correctly identifies the correct region boundary.  In the scenario with no noise, the variance of the algorithm's detected region boundary is small, regardless of the sample size.  For larger levels of noise, the variance decreases with the sample size, as expected.  
\begin{figure}
	\centering
	\includegraphics[width=0.9\textwidth]{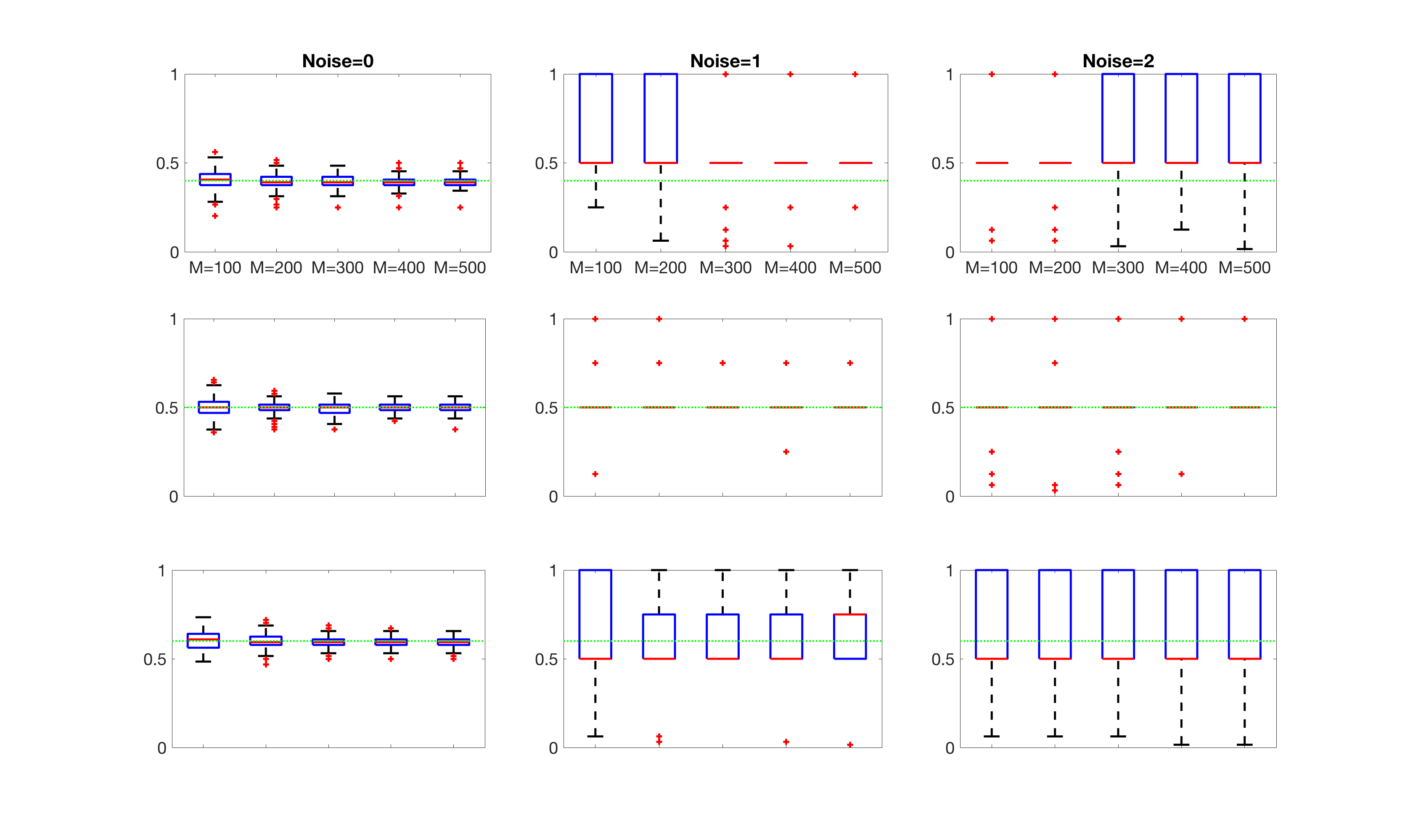}
	\caption{Region boundaries detected by Algorithm \ref{alg:cim} for various noise levels and sample sizes.  The hashed green line represents the actual region boundary, $r$, and the box and whisker plots represent the non-parametric distribution of the detected region boundary by Algorithm \ref{alg:cim}, for an $msi=\frac{1}{64}$ and $\alpha=0.2$.}
	\label{fig:region_detection_performance}
\end{figure}

Next, we investigate the sensitivity of Algorithm \ref{alg:cim} to the hyperparameter $msi$.  For various dependency types, we compute the maximum deviation of the $\textit{CIM}$ value over 500 Monte-Carlo simulations for sample sizes, $M$, ranging from 100 to 1000 for $msi$ taking on one the values of the set $\{\frac{1}{4}, \frac{1}{8}, \frac{1}{16}, \frac{1}{32}, \frac{1}{64}\}$, for $\alpha=0.2$.  Fig.~\ref{fig:algo_sensitivity} shows the maximum deviation of the estimated CIM value for each value of noise over sample sizes ranging from 100 to 1000 for eight different association patterns for these values of $msi$.  The results show that when the dependency is not masked by the $msi$ parameter, the algorithm's maximum deviation over the noise range, sample sizes, and dependencies tested is no greater than \num{4e-3}.  This is shown by the blue lines for the linear, quadratic, fourth-root, circular, and step function dependencies, and by the red lines in the cubic and sinusoidal dependencies.  When the dependency is masked by the $msi$, as expected, the algorithm is sensitive to the chosen value of $msi$.  As seen in Fig~\ref{fig:algo_sensitivity}, the maximum deviation of the algorithm for low-noise levels can reach a value close to $0.5$ for the low-frequency sinusoidal dependency.  From this, we can infer that small values of $msi$ should be chosen for more robust results for estimating $\widehat{CIM}$, as they empirically have minimal effect on measuring association patterns that do not have many regions of monotonicity, but have a positive effect on detecting and measuring dependencies with many regions of monotonicity.  The only drawback of choosing very small values of $msi$ is that they require more computational resources.

Next, we test the sensitivity of the algorithm to various values of $\alpha$.  More specifically, for the various dependence structures that are considered, we compute the maximum deviation of the $\textit{CIM}$ estimation over 500 Monte-Carlo simulations for sample sizes, $M$, ranging from 100 to 1000 for $\alpha$ taking on one of the values of the set $\{0.05, 0.10, 0.15, 0.20, 0.25, 0.30\}$, for $msi=\frac{1}{64}$.  Fig.~\ref{fig:algo_alpha_sensitivity} displays the maximum deviation of the estimated CIM value for each value of noise over sample sizes ranging from 100 to 1000 for eight different association patterns.  The results show that the algorithm is minimally sensitive to the value of $\alpha$ for dependencies that have more than a small number of monotonic regions, such as the sinusoidal dependencies.  This is easily explained by (\ref{eq:tau_convergence}) and (\ref{eq:decision_criterion}) which show that the upper bound of the variance of the $\tau$ estimate is high with small sample sizes and that the small number of samples combined with a large $\alpha$ prevent reliable detection of region boundaries.  

Finally, following Theorem \ref{thm:tau_convergence}, we demonstrate through simulations that Algorithm \ref{alg:cim} converges to the true \textit{CIM} value.  The results of the algorithm's convergence performance are displayed in Fig.~\ref{fig:algo_convergence}.  The subtitles for each subplot indicate the number of samples required such that the error between the estimated value, $\widehat{\textit{CIM}}$, and the true value, $\textit{CIM}$, over all computed noise levels is less than 0.01 over 500 Monte-Carlo simulations.  It can be seen that for the dependencies with small numbers of regions of monotonicity, Algorithm \ref{alg:cim} converges very quickly to the true value over all noise levels.  On the other hand, the dependencies with a large number of regions of monotonicity, such as the high frequency sinusoidal relationship depicted in the fifth subplot, a larger number of samples is required in order to ensure convergence.  This can be explained from the fact that the variance of the $\widehat{\textit{CIM}}$ increases as the number of samples decreases.  Thus, with a smaller number of samples in a dependency structure, the increased variance leads Algorithm \ref{alg:cim} to make incorrect decisions regarding the region boundaries.  As the number of samples increases, the detection performance increases.

\begin{figure}
	\centering
	\includegraphics[width=0.9\textwidth]{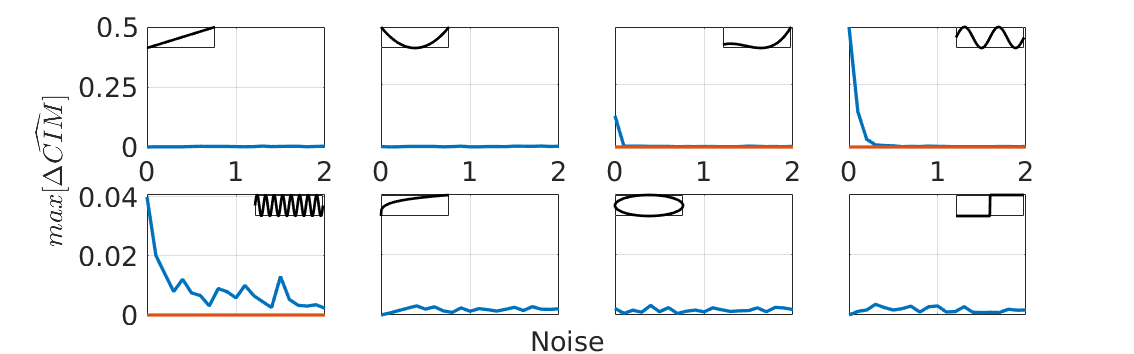}
	\caption{The maximum sensitivity of Algorithm \ref{alg:cim} for various association patterns (shown in the upper left inset) swept over different values of noise for sample sizes ($M$) ranging from 100 to 1000 and $msi$ taking on one of the values in the set $\{\frac{1}{4}, \frac{1}{8}, \frac{1}{16}, \frac{1}{32}, \frac{1}{64}\}$, with $\alpha=0.2$.  The red lines show the maximum sensitivity when the $msi$ value does not mask the dependence structure for the cubic and sinusoidal dependencies. }
	\label{fig:algo_sensitivity}
	
	\vspace*{\floatsep}
	
	\includegraphics[width=0.9\textwidth]{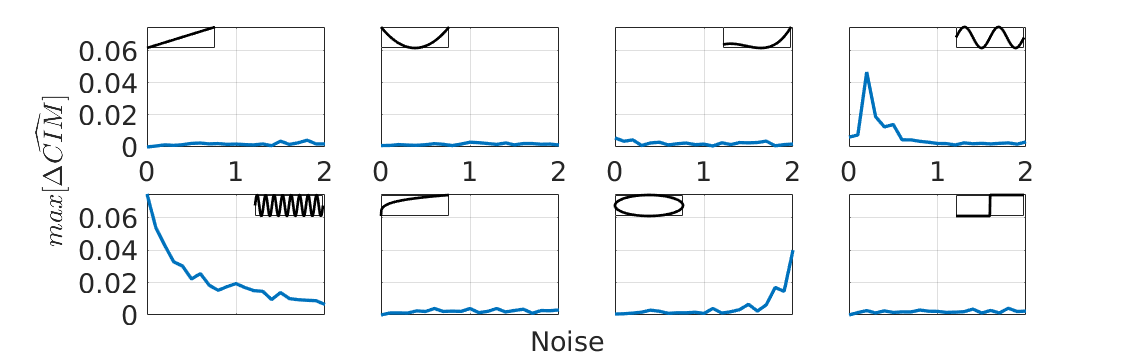}
	\caption{The maximum sensitivity of Algorithm \ref{alg:cim} for various association patterns (shown in the upper left inset) swept over different values of noise for sample sizes,$M$, ranging from 100 to 1000 and $\alpha$ taking on one of the values in the set $\{0.05, 0.10, 0.15, 0.20, 0.25, 0.30\}$, with $msi=\frac{1}{64}$.}
	\label{fig:algo_alpha_sensitivity}
	
	\vspace*{\floatsep}
	
	\includegraphics[width=0.9\textwidth]{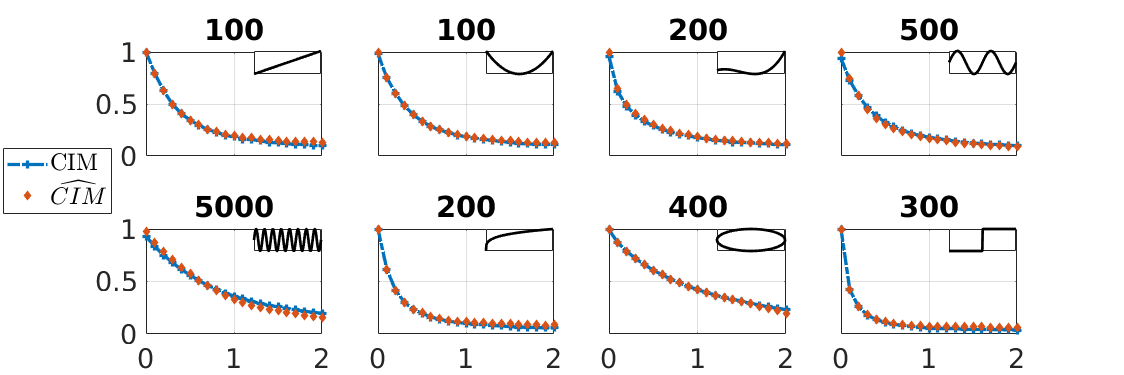}
	\caption{Theoretical and estimated values of \textit{CIM} for various association patterns shown in the upper right inset swept over different noise levels.  The subtitle shows the minimum number of samples for $\widehat{CIM}$ to be within 0.01 of \textit{CIM} over all noise levels tested for 500 Monte-Carlo simulations.  The simulations were conducted with $\alpha=0.2$ and $msi=\frac{1}{64}$.}
	\label{fig:algo_convergence}
	
\end{figure}

\subsubsection{Null Distribution of $\widehat{CIM}$}
The null distribution of $\widehat{CIM}$ can be theoretically modeled by

\begin{equation}\label{eq:cim_null_distrbution_theoretical}
    Z_{\widehat{CIM}} \sim \sum_i^R w_i |\mathcal{N}(0,\sigma_i)|,
\end{equation}
where $\sigma_i$ is the standard deviation of Kendall Tau's estimate for the given number of samples in region $i$, due to the asymptotic normality of the $\tau$ estimator, $R$ is the number of regions that were detected by the algorithm, and $w_i$ is the same as in (\ref{eq:cim_estimator}).  It is found experimentally that $Z_{\widehat{CIM}}$ can be approximated by the Beta distribution, as displayed in Fig.~\ref{fig:cim_hat_qqplot}.  Figs.~\ref{fig:cim_hat_qqplot} (b) and (c) both show that as the sample size increases, the $\alpha$ shape parameter remains relatively constant while the $\beta$ shape parameter increases linearly as a function of $M$.  This roughly corresponds to a distribution converging to a delta function centered at zero.  This is a desirable property because it implies that \textit{CIM} approximation algorithm yields a value close to $0$ for data drawn from independent random variables with a decreasing variance as the number of samples used to compute $\textit{CIM}$ increases.  The error between the Beta approximation and the true distribution in (\ref{eq:cim_null_distrbution_theoretical}) is characterized in Fig.~\ref{fig:nulldist_approx_error}.  Here, the error is captured as the Hellinger distance between the true distribution, under the assumption of two regions being detected (which was empirically tested across the sample sizes), and the Beta approximation with parameters defined in Fig.~\ref{fig:cim_hat_qqplot} (b) and (c).  It is seen that as the sample size increases, the distance between the true and approximate distributions decreases.

\begin{figure}
	\centering
	\includegraphics[width=1\textwidth]{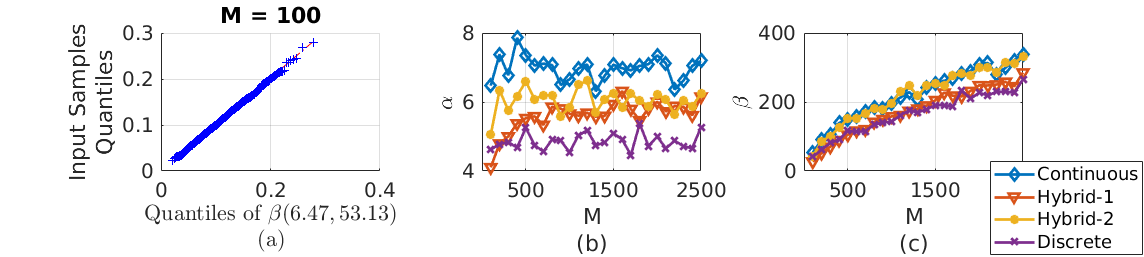}
	\caption{(a) QQ-Plot of $CIM$ for continuous random variables $X$ and $Y$ such that $X \bigCI Y$ and $M=100$, (b) $\alpha$ of the distribution of \textit{CIM} as a function of $M$, (c) $\beta$ of the distribution of \textit{CIM} as a function of $M$ }
	\label{fig:cim_hat_qqplot}
\end{figure}

\begin{figure}
    \centering
    \includegraphics[width=0.5\textwidth]{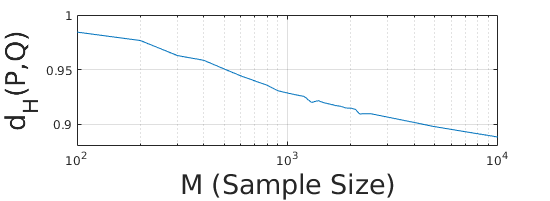}
    \caption{Hellinger Distance between the true null distribution for detection of two regions as given by (\ref{eq:cim_null_distrbution_theoretical}), and the approximated Beta distribution}
    \label{fig:nulldist_approx_error}
\end{figure}

\subsubsection{Computational Complexity}
In this section we describe the computational complexity of computing \textit{CIM} algorithm above.  We propose a new algorithm to compute $\hat{\tau}_{KL}$ to achieve a computational complexity of $\mathcal{O}(n^2)$ when estimating \textit{CIM} for continuous and discrete random variables, and $\mathcal{O}(n^3)$ when estimating \textit{CIM} for hybrid random variables.  

The core of Algorithm \ref{alg:cim}, described earlier, consists of repeated computations of $\hat{\tau}_{KL}$.  If one were to n{\"a}ively compute this, by recomputing the number of concordant and discordant pairs every time a new region was tested, the operations required to compute the number of concordant and discordant samples would exponentially increase.  Instead, we propose another algorithm to compute $\hat{\tau}_{KL}$ efficiently while accumulating new samples into the batch of data for which the value of $\hat{\tau}_{KL}$ is desired (i.e. when expanding the region by $si$).  The essence of this algorithm is that it pre-sorts the data in the direction being scanned, so that the number of concordant and discordant samples do not need to be recomputed in every iteration of the scanning process.  Instead, the sorted data allows us to store in memory the number of concordant and discordant samples, and update this value every time a new sample is added to the batch of samples being processed.  Additionally, during the sorting process, the algorithm converts floating point data to integer data by storing the statistical ranks of the data rather than the data itself, allowing for potentially efficient FPGA or GPU based implementations.  The efficient algorithm to compute $\hat{\tau}_{KL}$ for continuous and discrete data, given a new sample, is described in the \textproc{consume} function of Algorithm \ref{alg:taukl}, which is detailed in \ref{alg2_appendix}.  From Algorithm \ref{alg:taukl}, it is seen that if $n$ samples are to be processed, then the \textproc{consume} function is called $n$ times.  For clarity of exposition, the remaining helper functions are not presented; however, their operation is only to initialize the variables.

The \textproc{consume} function has a computational complexity of $\mathcal{O}(n)$, due to lines 9 and 10 in Algorithm \ref{alg:taukl}, which require computation over a vector of data.  The consume function is called $n$ times by Algorithm \ref{alg:cim} in order to process all the samples, yielding a total complexity of $\mathcal{O}(n^2)$.  It should be noted that lines 9 and 10 in Algorithm \ref{alg:taukl} are vectorizable operations, and the initial presorting is an $\mathcal{O}(n log(n))$ operation.  For hybrid data, additional calculations are required in the \textproc{consume} function in order to count the number of overlapping samples between discrete outcomes in the continuous domain, as described in (\ref{eq:taukl}).  This requires an additional $\mathcal{O}(n)$ operations, bringing the computational complexity of Algorithm \ref{alg:cim} to process hybrid random variables to $\mathcal{O}(n^3)$.  For clarity, the pseudocode to compute the overlapping points is not shown in Algorithm \ref{alg:taukl}, but a reference implementation to compute $\hat{\tau}_{KL}$ is provided\footnote{\url{https://github.com/stochasticresearch/depmeas/blob/master/algorithms/taukl_s.m}}.

\section{Simulations} \label{sec:exp}
In this section, we compare \textit{CIM} to other metrics of dependence and analyze their performance.  We begin by conducting synthetic data experiments to understand the bounds of performance for all state-of-the-art dependence metrics.  We then apply \textit{CIM} to real world datasets from various disciplines of science, including computational biology, climate science, and finance.

\subsection{Synthetic Data Simulations} \label{sec:synth_experiments}

Following \cite{simonandtibs}, we begin by comparing the statistical power of \textit{CIM} against various estimators of mutual information, including k-nearest neighbors estimation \cite{knn_mi}, adaptive partitioning MI estimation \cite{shannonapMI}, and MI estimation based on von Mises expansion \cite{vonMisesMI}.  The motivation for this simulation stems from Section \ref{sec:dpi}, where it was proved that \textit{CIM} satisfied the DPI and thus, could be substituted for measures of mutual information.    Fig.~\ref{fig:cim_vs_ite} compares these metrics and shows that \textit{CIM} outperforms the compared estimators of mutual information for all dependency types considered \footnote{The source code for Shannon Adaptive Partitioning and von Mises based MI estimators is from the ITE Toolbox \cite{itetoolbox}.  K-NN based MI estimation source code is from \url{https://www.mathworks.com/matlabcentral/fileexchange/50818-kraskov-mutual-information-estimator}}.  The results displayed in Fig.~\ref{fig:cim_vs_ite} are from simulations with a sample size of $M=500$.  Although we do not include additional plots here, even for small sample sizes such as $M=100$ (which are typical for biological datasets where estimators of the \textit{MI} are commonly used), \textit{CIM} outperforms the compared estimators of \textit{MI} for all the association patterns tested.  These simulations suggest that \textit{CIM} can indeed replace estimators of the \textit{MI} when used with algorithms which rely on the DPI, such as ARACNe \cite{aracne} or MRNET \cite{mrnet}.  

\begin{figure}
	\centering
	\includegraphics[width=1\textwidth]{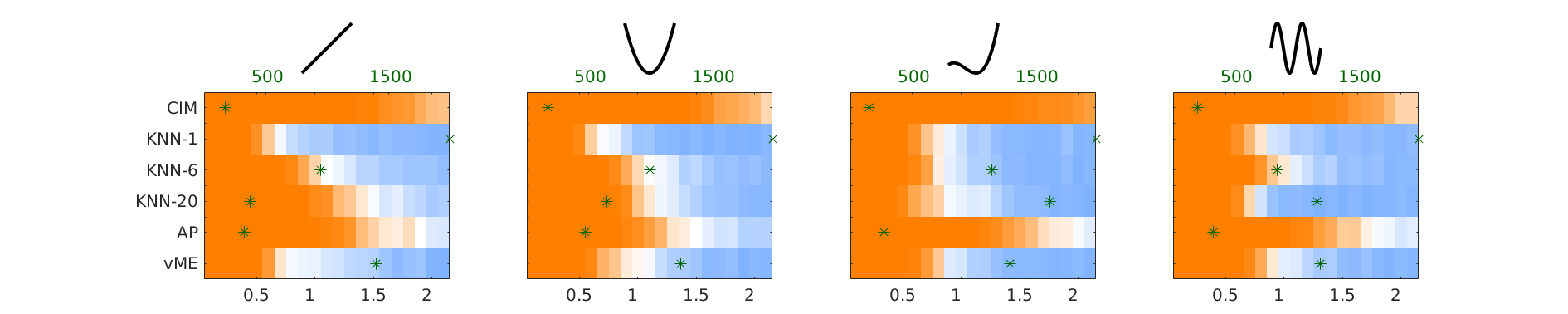}
	\includegraphics[width=1\textwidth]{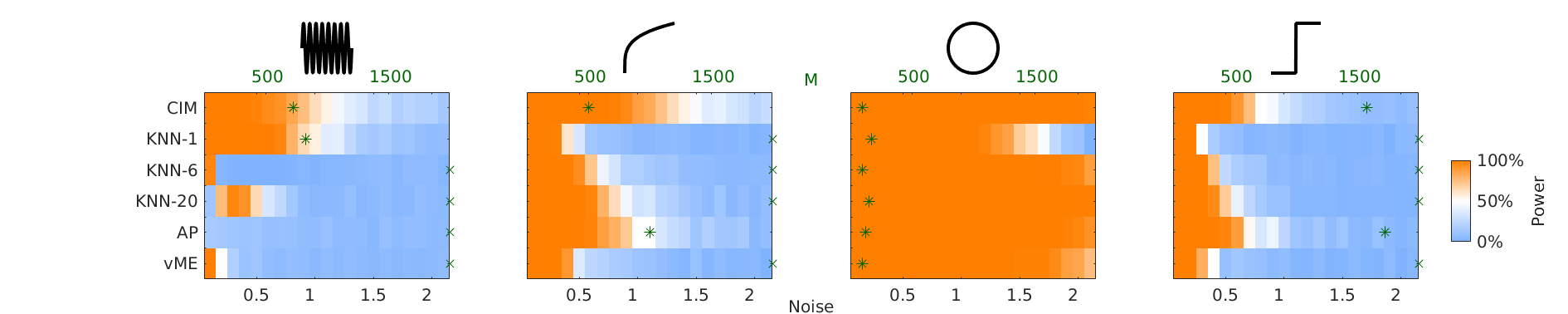}
	\caption{Statistical power of \textit{CIM} and various estimators of mutual information including the KNN-1, the KNN-6, the KNN-20, Adaptive Partitioning, and von Mises Expansion for sample size $M = 500$ and computed over 500 Monte-Carlo simulations. Noise-free form of each association pattern is shown above each corresponding power plot.  The green asterisk displays the minimum number of samples required to achieve a statistical power of $0.8$ for the different dependency metrics considered for a noise level of $1.0$.  A green $\times$ symbol is shown if the number of samples required is beyond the scale of the plot.}
	\label{fig:cim_vs_ite}
\end{figure}

We also investigate the power characteristics of \textit{CIM} and estimators of mutual information as a function of the sample size. The green asterisk in  Fig.~\ref{fig:cim_vs_ite} displays the minimum number of samples required to achieve a statistical power of $0.8$ for the different dependency metrics considered for a noise level of $1.0$.  A green $\times$ symbol is shown if the number of samples required is beyond the scale of the plot.  It is seen that \textit{CIM} outperforms the compared estimators for all dependency types considered.  In general, \textit{CIM} displays good small sample performance because it is based on Kendall's $\tau$, which is shown to have superior small sample performance as compared to other metrics of monotonic dependence \cite{bonett2000, USGSbook}.

Next, we compare \textit{CIM} to other state-of-the-art dependence metrics, which are not proven to satisfy the DPI.  We begin by comparing the estimated indices for various functional and stochastic dependencies for continuous and discrete marginals.  The results, displayed in Fig.~\ref{fig:wikipedia}, show that \textit{CIM} performs equivalently to other leading measures of dependence, including $MIC_e$, the \textit{RDC}, the \textit{dCor}, the \textit{Ccor}, and \textit{CoS} for continuous and discrete random variables in the absence of noise.  \textit{CIM} achieves $+1$ for all functional dependencies with continuous marginals (Fig.~\ref{fig:wikipedia} (a), (c)) and for monotonic functional dependencies with discrete marginals (Fig.~\ref{fig:wikipedia} (a), (c)), and values close to $+1$ for nonmonotonic functional dependencies with discrete marginals (Fig.~\ref{fig:wikipedia} (d), (e), (f)).  Only the \textit{RDC} shows similar performance.  However, as shown in Fig.~\ref{fig:wikipedia} (b) and (e), the \textit{RDC} has the highest bias in the independence case.  Discrete random variables are not tested for the \textit{Ccor} and \textit{CoS} metrics because they were not designed to handle discrete inputs.  

Fig.~\ref{fig:cim_vs_dep} compares the statistical power of \textit{CIM} to other state-of-the-art dependence metrics which are not proven to satisfy the DPI.  The results in Fig.~\ref{fig:cim_vs_dep} show that \textit{CIM} displays the best performance for quadratic, cubic, and sinusoidal dependence.  For linear, fourth-root, and step function dependence, it performs better than the \textit{RDC}, \textit{TIC}, and the \textit{Ccor}, but is beaten by \textit{CoS} and the \textit{dCor}.  In the high frequency sinusoidal case, it is more powerful than the \textit{RDC} but less powerful than the \textit{TIC}.  This can be explained by the fact that the region configuration which maximizes the dependence (\textit{lines 25-26} in Algorithm \ref{alg:cim}) becomes more ambiguous as the noise level increases when multiple partitions of the range space of $X-Y$ are needed.  Our general observations are that \textit{CoS} and the \textit{dCor} are the best for monotonic dependencies, \textit{CIM} is the best for small numbers of monotonic regions, and \textit{TIC} performs extremely well for high frequency sinusoidal dependencies.  The sample size requirements, again shown with the green asterisk and plus symbols, reflect these same observations.

\footnotetext{The code for these dependency metrics and simulations is provided here: \url{https://github.com/stochasticresearch/depmeas}}

\begin{figure}
    \centering
	\begin{subfigure}[b]{0.20\textwidth}
		\includegraphics[width=\linewidth]{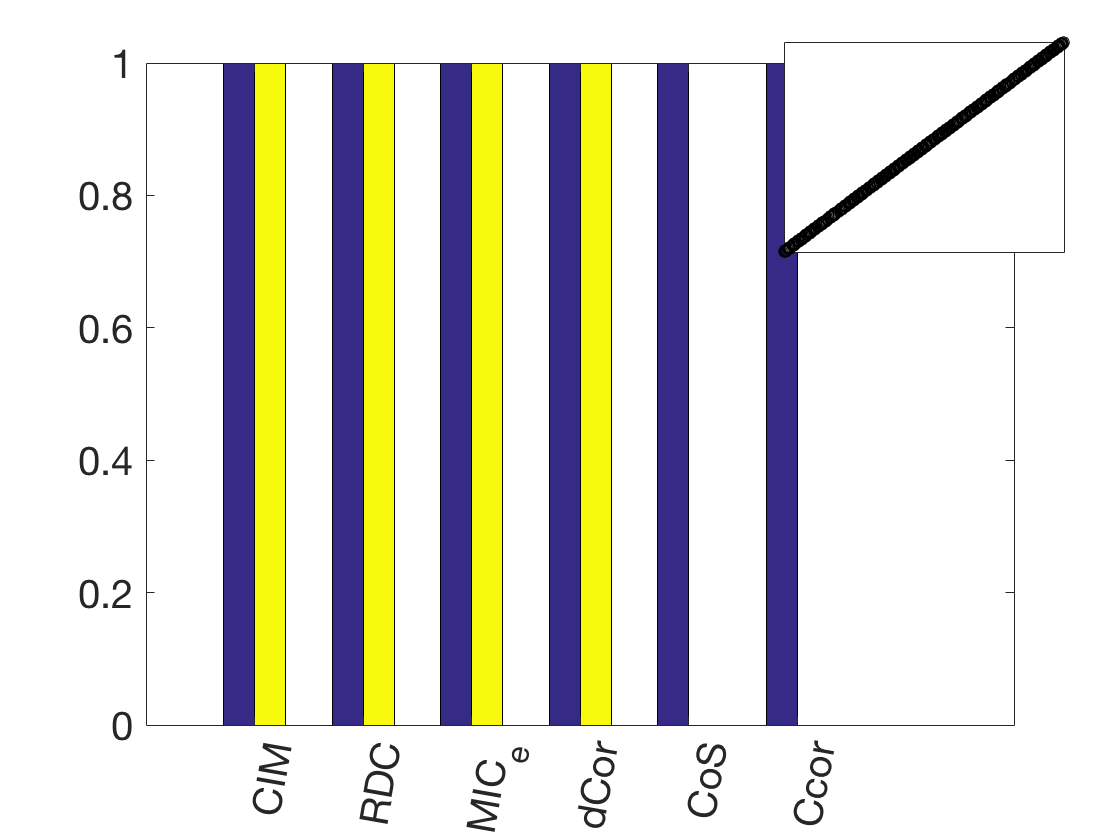}
		\caption{}
	\end{subfigure}%
	\begin{subfigure}[b]{0.20\textwidth}
		\includegraphics[width=\linewidth]{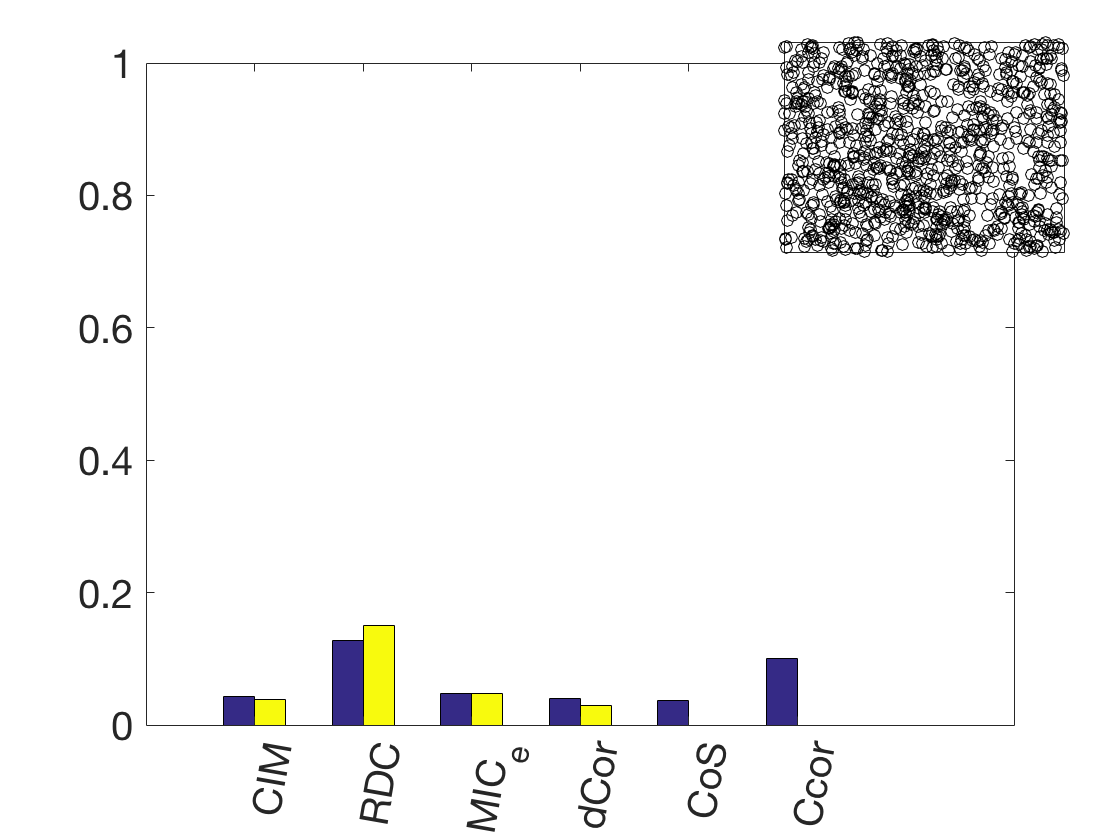}
		\caption{}
	\end{subfigure}%
	\begin{subfigure}[b]{0.20\textwidth}
		\includegraphics[width=\linewidth]{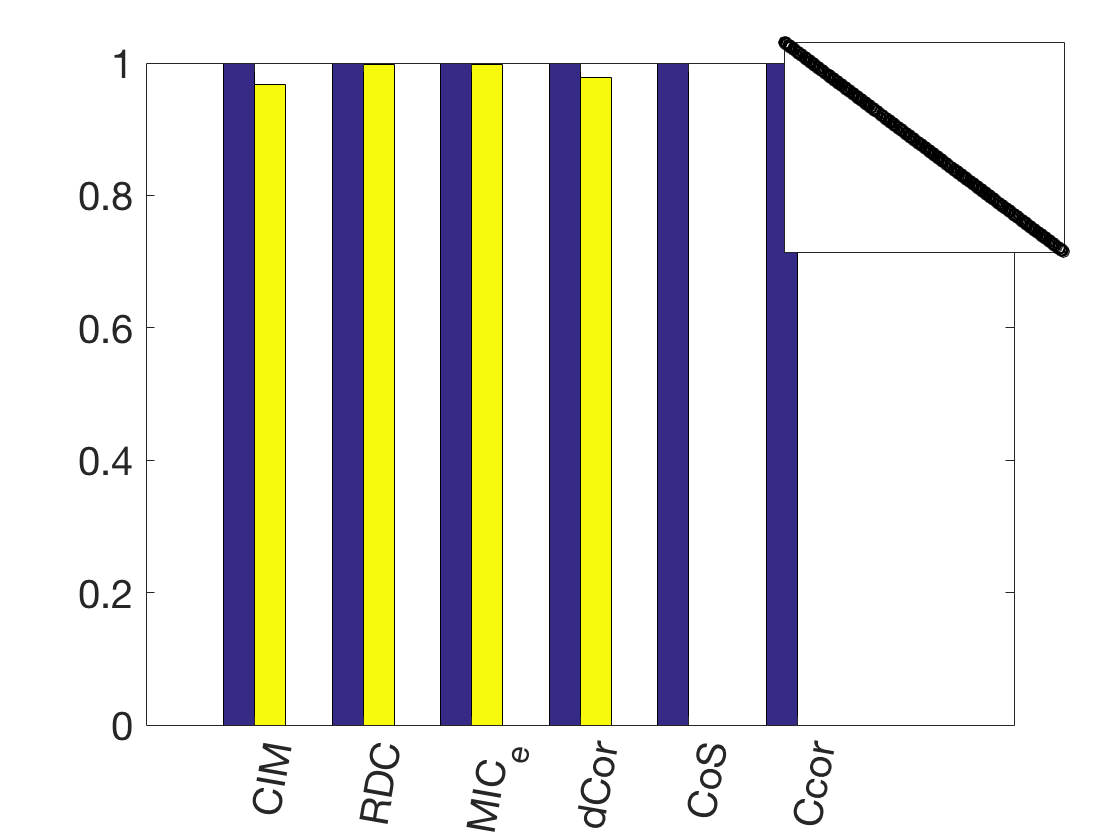}
		\caption{}
	\end{subfigure}%
	\begin{subfigure}[b]{0.20\textwidth}
		\includegraphics[width=\linewidth]{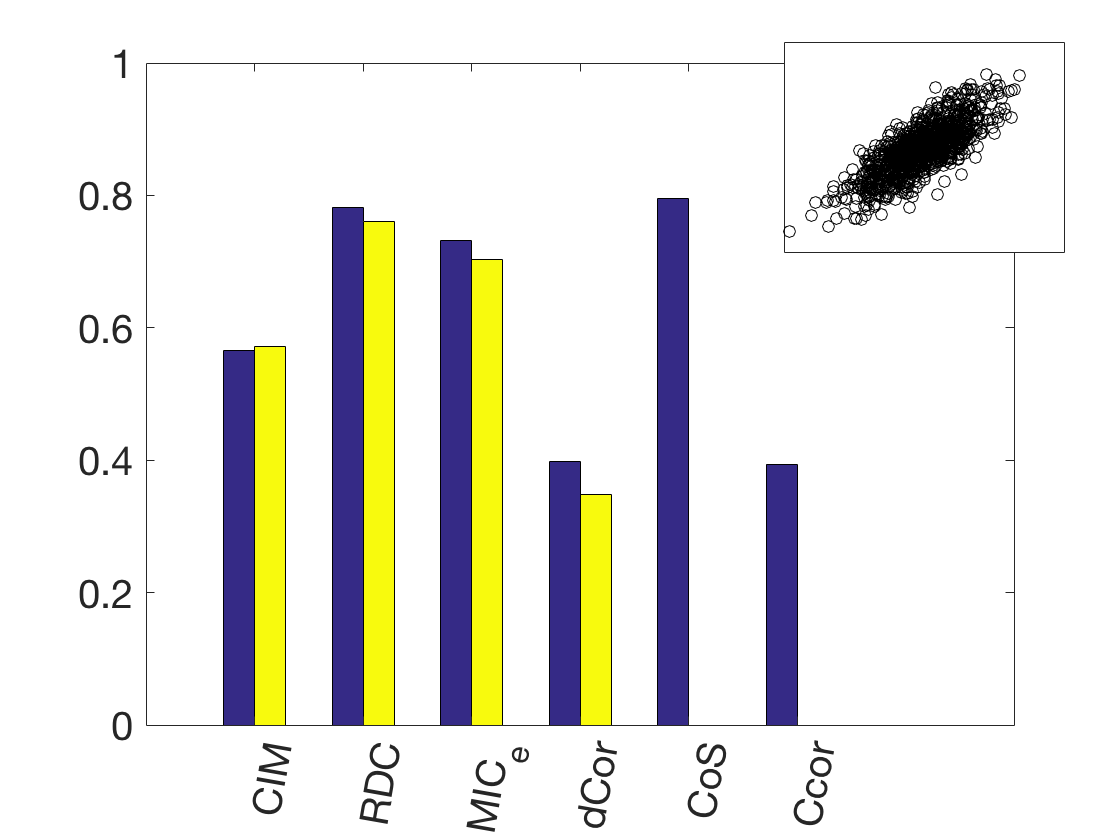}
		\caption{}
	\end{subfigure}\
	\begin{subfigure}[b]{0.20\textwidth}
		\includegraphics[width=\linewidth]{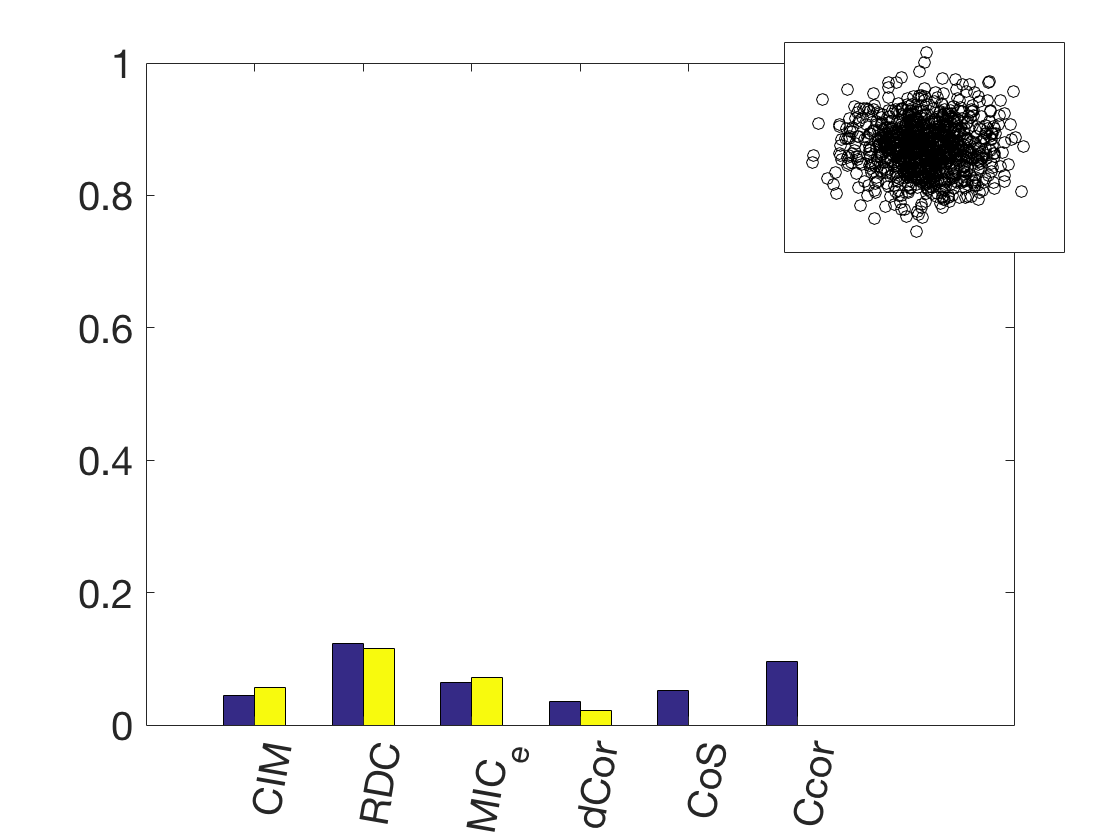}
		\caption{}
	\end{subfigure}%
	\begin{subfigure}[b]{0.20\textwidth}
		\includegraphics[width=\linewidth]{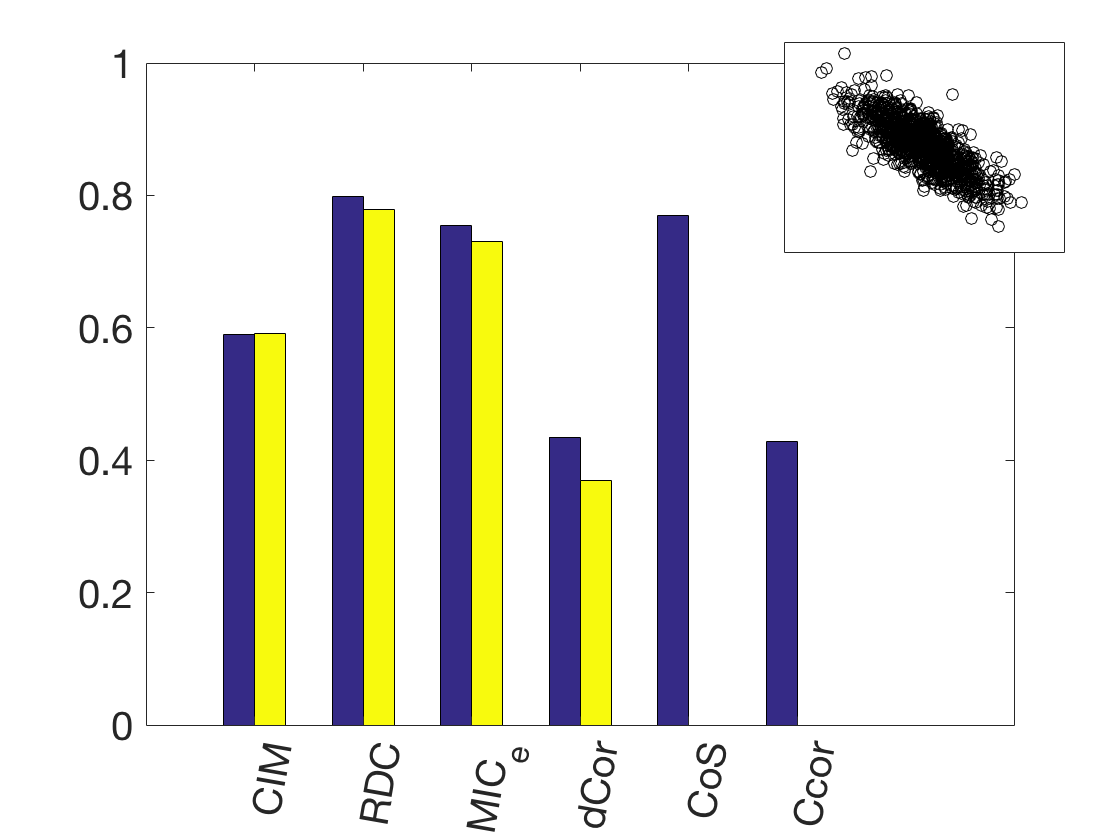}
		\caption{}
	\end{subfigure}%
	\begin{subfigure}[b]{0.20\textwidth}
		\includegraphics[width=\linewidth]{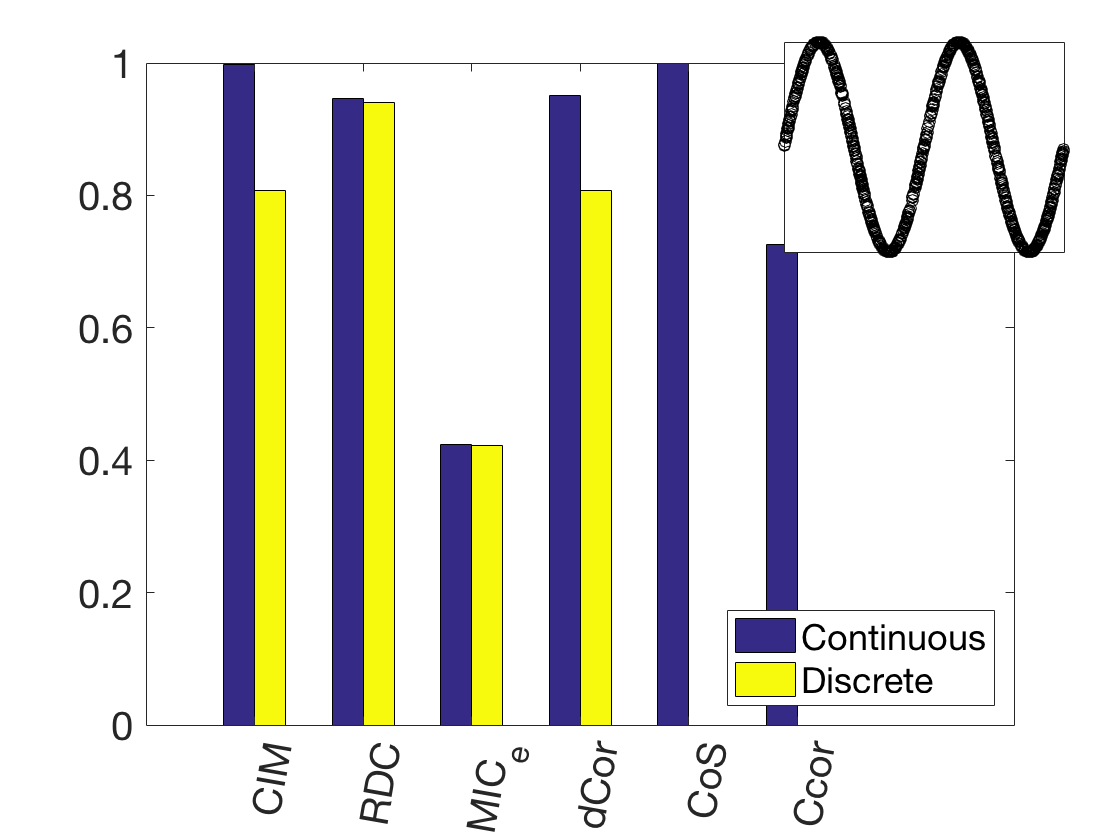}
		\caption{}
	\end{subfigure}%
	\begin{subfigure}[b]{0.20\textwidth}
		\includegraphics[width=\linewidth]{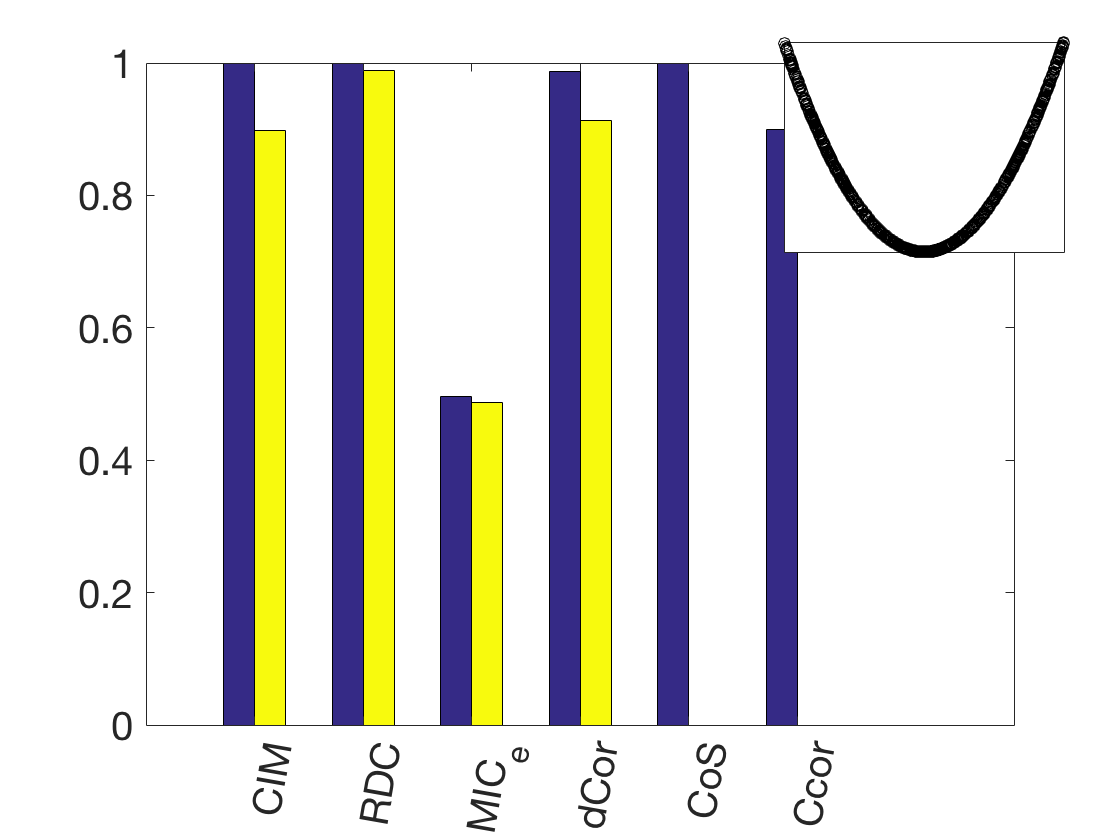}
		\caption{}
	\end{subfigure}\
	\caption{Values attained by various dependence metrics for various noiseless functional associations (a),(c),(g),(h), and (i) and Gaussian copula associations (d), (e), and (f).  (b) is the independence case, and (e) is the Gaussian copula with $\rho=0$.}
	\label{fig:wikipedia}
\end{figure}

\begin{figure}
	\centering
	\includegraphics[width=1\textwidth]{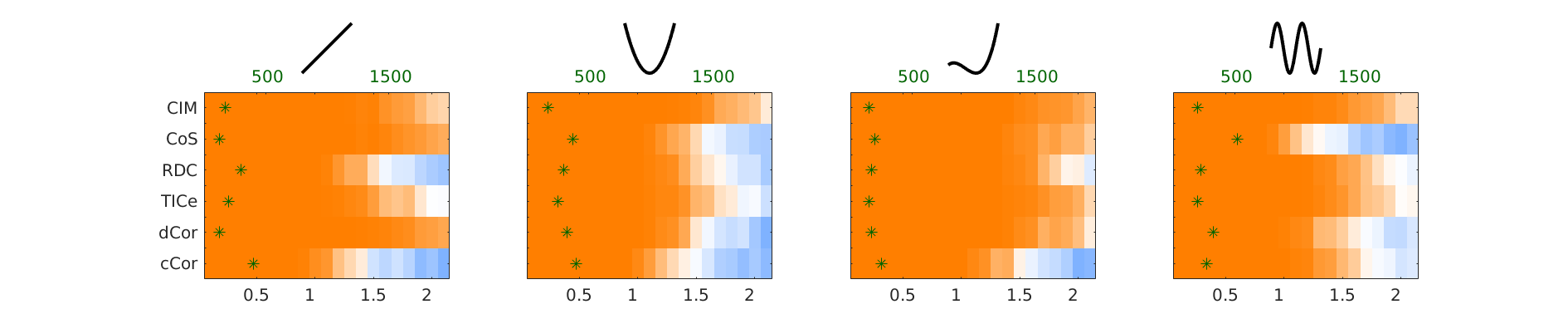}
	\includegraphics[width=1\textwidth]{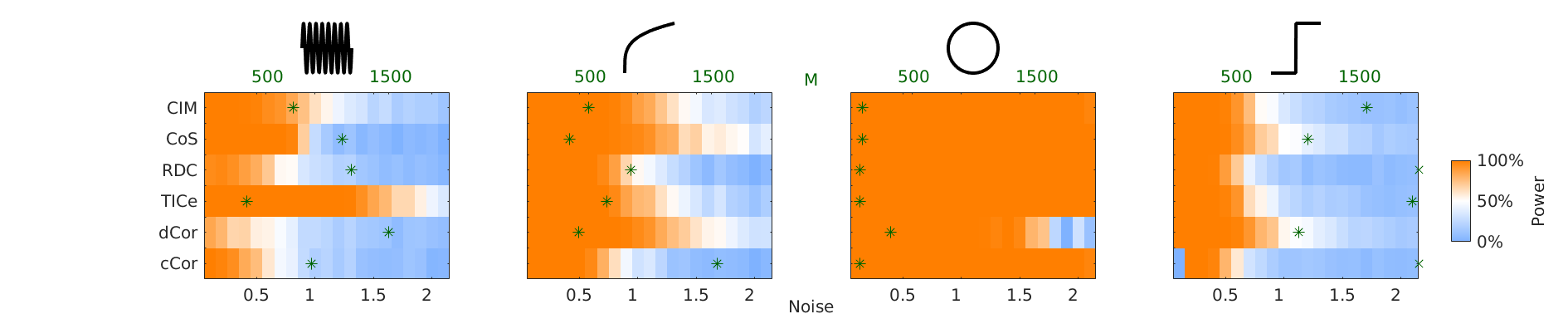}
	\caption{Statistical power of \textit{CIM} and various measures of dependence including CoS, the RDC, TICe, the dCor, and the cCor for sample size $M = 500$ and computed over 500 Monte-Carlo simulations. Noise-free form of each association pattern is shown above each corresponding power plot.  The green asterisk displays the minimum number of samples required to achieve a statistical power of $0.8$ for the different dependency metrics considered for a noise level of $1.0$.  A green plus symbol is shown if the number of samples required is beyond the scale of the plot.}
	\label{fig:cim_vs_dep}
\end{figure}

\subsection{Real Data Simulations} \label{sec:cim_real_data}
In this section, we apply \textit{CIM} metric to various data exploration and machine learning problems using real-world data, including discovering interesting dependencies in large datasets, Markov network modeling, stochastic modeling of random variables.  

\subsubsection{Data Exploration}
We begin by applying \textit{CIM} metric to real data with the primary goal of characterizing the monotonicity structure of data from different areas of science.  This is motivated by both the fields of joint probabilistic data modeling and data exploration.  More explicitly, for joint probabilistic modeling of high dimensional datasets, many copula-based techniques are beginning to be adopted in practice, including copula Bayesian networks \cite{cbn} and vine copula models \cite{vinecopula} due to their flexibility in modeling complex nonlinear relationships between random variables.  The authors of these methods advocate the use of parametric copula families for modeling local joint probability distributions. The main reason for this is that it is computationally efficient to estimate a parametric copula for a joint dataset using the relationship between the copula parameter, $\theta$, and a measure of concordance such as Kendall's $\tau$.  However, popular copula families such as the Archimedean and Gaussian families only capture monotonic dependencies.  Thus, if datasets being modeled are nonmonotonic, these copulas will fail to model all the dynamics of the underlying data.  Conversely, if the dependencies within these datasets are monotonic, these efficient procedures can be used and to fit the data to known copula families, and computationally expensive techniques such as estimating empirical copulas can be ignored.  Thus, to know whether a parametric copula family can be used, the monotonicity structure must be understood.  Therefore, from a copula modeling and analysis perspective, knowledge of the monotonicity structure provides more actionable information than Reshef's proposed nonlinearity coefficient, defined as
\begin{equation}\label{eq:reshefnonlinearity}
\theta_{Reshef} = MIC - \rho,
\end{equation}
where $\rho$ is the Pearson's correlation coefficient \cite{corrcoef}.  Interestingly, copulas can capture monotonic nonlinear relationships while the nonlinearity coefficient defined in (\ref{eq:reshefnonlinearity}).

In order to answer these questions, we process pairwise dependencies for multiple datasets related to gene expression data, financial returns data, and climate features data\footnote{Details of the datasets used and how they were processed are provided in \ref{appendix:datadetails}}.  For each pairwise dependency within a dataset, we count the number of monotonic regions by examining the number of regions detected by Algorithm \ref{alg:cim}.  Additionally, to prevent overfitting, we decide that a pairwise dependency only has one monotonic region if the value of $\hat{\tau}_{KL}$ is within 5 \% of the estimated value of \textit{CIM}.  When time-series data is compared, we only include results of dependencies where the data is considered stationary by the Dickey-Fuller test, at a significance level of $\alpha = 0.05$, and ensure time coherency between the series being compared.  Due to the CIM's reliance on copulas, the only requirement is that the data be identically distributed; independence between samples is not required because a copula can capture both inter-dependence and serial dependence within realizations of a random variable.   Additionally, we only count dependencies if the dependence metric is statistically significant at a level of $\alpha = 0.05$ and the dependence strength exceeds a value of $0.4$ as measured by \textit{CIM} estimation algorithm.  Dependencies are only calculated for all unique combinations of features \textit{within} each dataset.  With these procedures, after processing 7765 pairwise dependencies which meet the criterion above for various cancer datasets, we find that 96\% of gene expression indicators within a cancer dataset are in fact monotonic.  Similarly, we process 73 pairwise dependencies between closing price returns data for 30 major indices over a period of 30 years.  We find that 99\% of the dependencies are monotonic.  Finally, we process over 42185 pairwise dependencies of El-Nino indicators in the Pacific ocean, land temperatures of major world cities over the past 200 years, and air quality indicators in major US cities in the past 15 years.  In these datasets, termed climate related datasets, we find that 97\% of the dependencies within each dataset that meet the criterion above are monotonic.  The prevalence of monotonicity in these datasets suggests that techniques that use copula modeling with popular copula families such as the Gaussian or Archimedean families will tend to capture the underlying dynamics of the data properly.

Conversely, \textit{CIM}'s unique ability to identify regions of monotonicity can be used to identify ``interesting'' dependence structures that may warrant closer analysis by subject matter experts.  As an example, Fig.~\ref{fig:temp_nonmonotoinc} shows a nonmonotonic association pattern that was automatically discovered by \textit{CIM} between temperature patterns in Andorra and Burkina Faso, while mining over $27,000$ pairwise dependencies.  As shown in Fig.~\ref{fig:temp_nonmonotoinc}, the time-series patterns do not clearly reveal this nonmonotonic dependency structure.  This example serves to highlight the ability of \textit{CIM} to discover these kinds of dependence structures automatically.

\begin{figure}
	\centering
	
	\begin{subfigure}[t]{0.25\textwidth}
		\centering
		\includegraphics[width=1\linewidth]{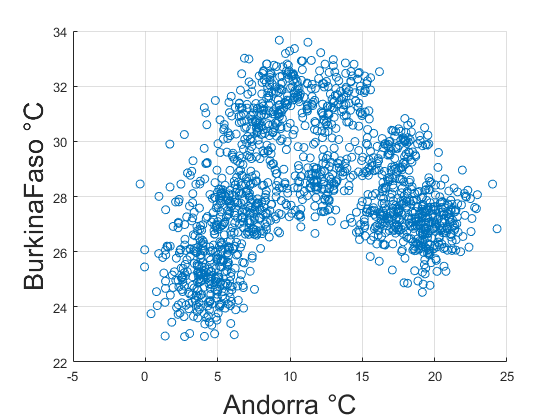}
		\caption{}
		\label{fig:ltdp1}
	\end{subfigure}
	\begin{subfigure}[t]{0.25\textwidth}
		\centering
		\includegraphics[width=1\linewidth]{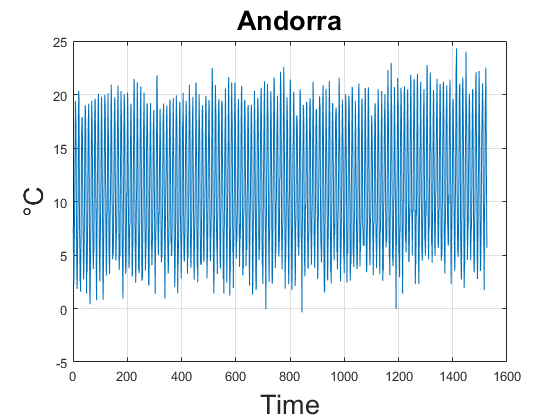}
		\caption{}
		\label{fig:ltdp2}
	\end{subfigure}
	\begin{subfigure}[t]{0.25\textwidth}
		\centering
		\includegraphics[width=1\linewidth]{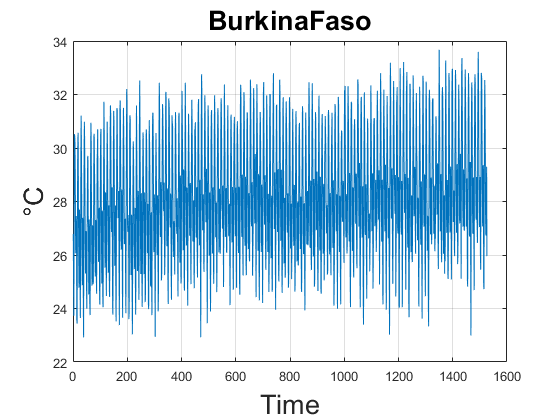}
		\caption{}
		\label{fig:ltdp3}
	\end{subfigure}
	\caption{(a) Scatter plot of time-aligned temperature data from Andorra and Burkina Faso, which reveals a nonmonotonic association pattern (b) Time-series of the temperature data from Andorra (c) Time-series of the temperature data from Burkina Faso.}
	\label{fig:temp_nonmonotoinc}
	
	\vspace*{\floatsep}
	
	\begin{subfigure}[t]{0.25\textwidth}
		\centering
		\includegraphics[width=1\linewidth]{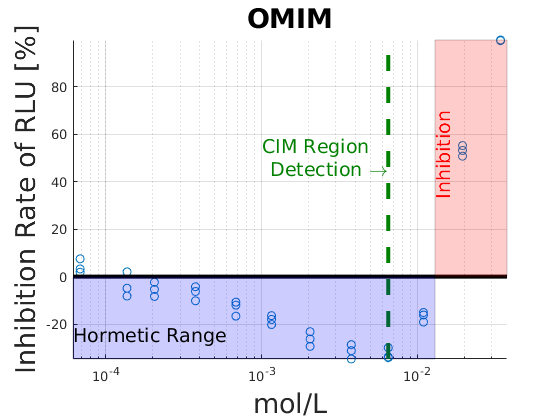}
		\caption{}
		\label{fig:nmdr1}
	\end{subfigure}
	\begin{subfigure}[t]{0.25\textwidth}
		\centering
		\includegraphics[width=1\linewidth]{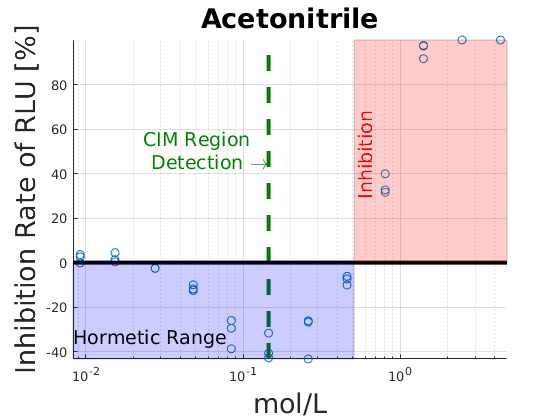}
		\caption{}
		\label{fig:nmdr2}
	\end{subfigure}
	\begin{subfigure}[t]{0.25\textwidth}
		\centering
		\includegraphics[width=1\linewidth]{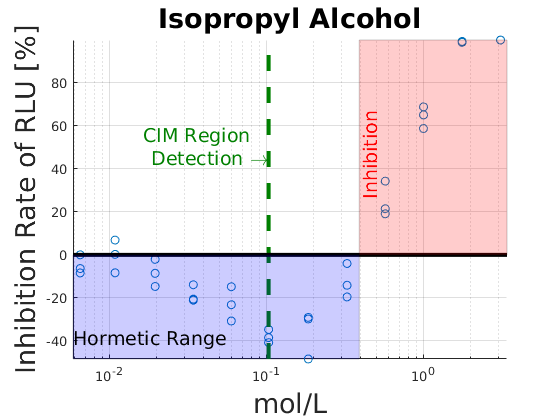}
		\caption{}
		\label{fig:nmdr3}
	\end{subfigure}
	\caption{(a) The hormetic effect of 1-octyl-3-methylimidazolium chloride ([OMIM]Cl, CAS RN. 64697-40-1) on firefly luciferase after 15 min exposure (b) the hormetic effect of acetonitrile (CAS RN. 75-05-8) on photobacteria Vibro-qinghaiensis sp. Q67 after 15 min exposure, and (c) the hormetic effect of NaBF4 (CAS RN.13755-29-8) on Vibro-qinghaiensis sp. Q67 after 12 h exposure.  The blue and red regions indicate the hormetic and inhibition regions of the dependence structure, respectively, as indicated by toxicological experts.  The green hashed line indicates the region boundary, detected by \textit{CIM} algorithm.}
	\label{fig:nmdr}
	
	\vspace*{\floatsep}
	
	\centering
	\begin{subfigure}[t]{0.25\textwidth}
		\centering
		\includegraphics[width=1\linewidth]{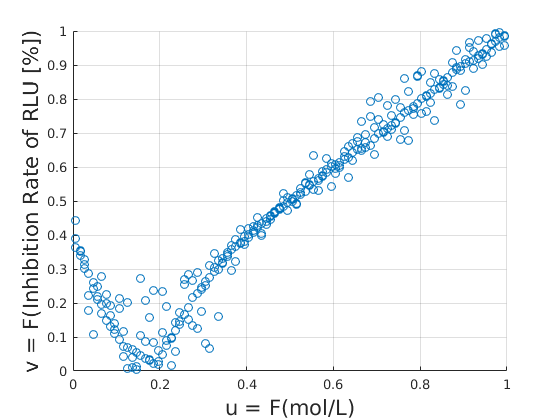}
		\caption{}
		\label{fig:nonmonotonic_cop_model1}
	\end{subfigure}
	\begin{subfigure}[t]{0.25\textwidth}
		\centering
		\includegraphics[width=1\linewidth]{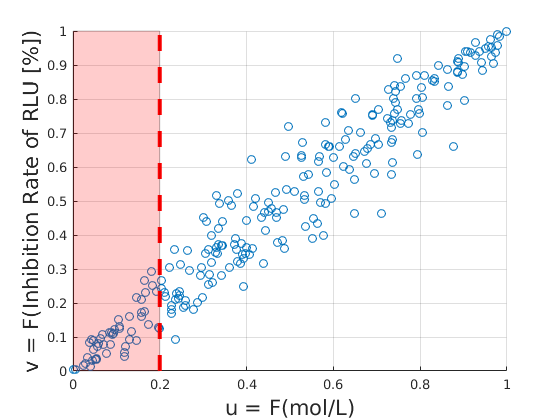}
		\caption{}
		\label{fig:nonmonotonic_cop_model2}
	\end{subfigure}
	\begin{subfigure}[t]{0.25\textwidth}
		\centering
		\includegraphics[width=1\linewidth]{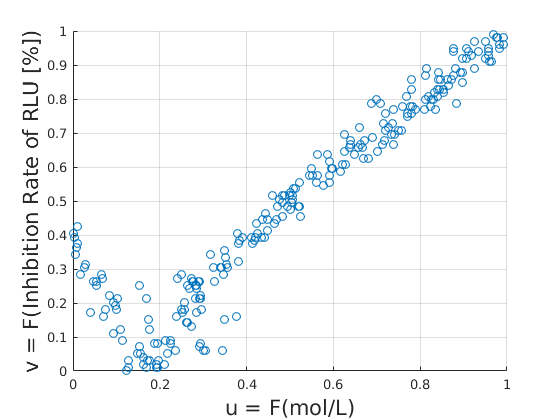}
		\caption{}
		\label{fig:nonmonotonic_cop_model3}
	\end{subfigure}
	\caption{(a) \textbf{OMIM} data from Fig.~\ref{fig:nmdr1}, interpolated with noise to provide more data-points for modeling purposes.  (b) Pseudo-Observations of a Gaussian copula model of data in (a).  The red highlighted region represents pseudo-observations which are incorrectly modeled by the Gaussian copula model. (c) Pseudo-Observations of an empirical copula model of data in (a).}
	\label{fig:nonmonotonic_cop_model}
	
\end{figure}


\subsubsection{Stochastic Modeling}
To highlight the importance of nonmonotonic dependence structures from a stochastic modeling perspective, we examine nonmonotonic dose response data from \citet{nmdr_reference}.  The data are displayed in Fig.~\ref{fig:nmdr}, and regions of importance of the relationship between the data as labeled by scientific experts in the field of toxicology is highlighted in the blue and pink regions.  Additionally, the unique ability of \textit{CIM} to automatically identify these regions is shown by the hashed green line.  The regions detected by \textit{CIM} correspond to where the monotonicity changes in the dependence structure.


To understand why regions of monotonicity are important from a data modeling perspective, we take the \textbf{OMIM} data from Fig.~\ref{fig:nmdr1} and show the difference between modeling it with a Gaussian copula and an empirical copula.  Fig.~\ref{fig:nonmonotonic_cop_model2} shows pseudo-observations drawn from a Gaussian copula model of the data displayed in Fig.~\ref{fig:nmdr1}, which are used to estimate the empirical copula model shown in Fig.~\ref{fig:nonmonotonic_cop_model3}.  The red highlighted region represents pseudo-observations that are incorrectly modeled by the Gaussian copula model.  This problem will occur with any copula model which captures only monotonic dependence structures, including for example, the popular parametric Archimedean family of models.  

This problem will in fact occur with any popular copula model, including copulas from the Archimedean family, due to the fact that the latter only capture monotonic dependence structures.  



\begin{figure}
	\centering
	\begin{subfigure}{0.3\textwidth}
		\centering
		\includegraphics[width=1\linewidth]{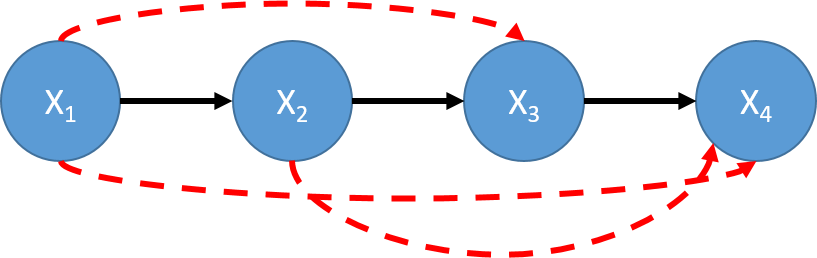}
		\caption{}
		\label{fig:network}
	\end{subfigure}\hspace{0.1\textwidth}
	\begin{subfigure}{0.3\textwidth}
		\centering
		\includegraphics[width=1\linewidth]{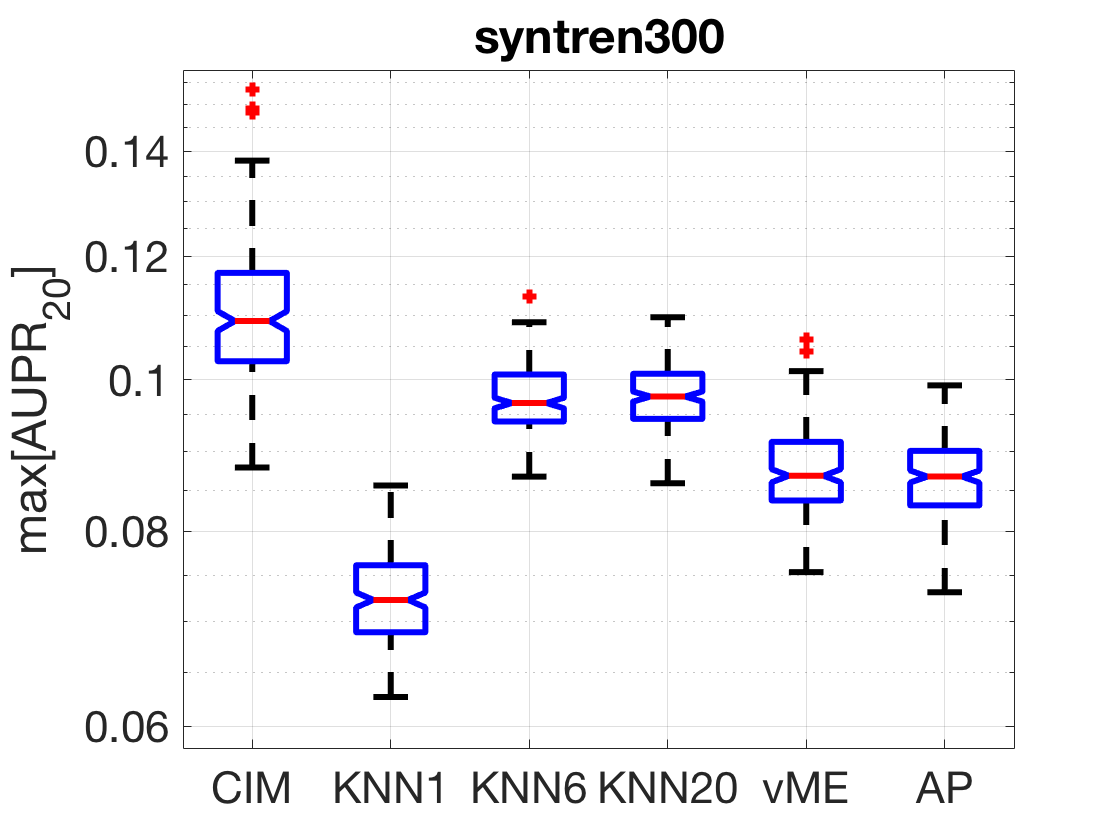}
		\caption{}
		\label{fig:mrnet_all_datasets}
	\end{subfigure}
	\caption{(a) The true Markov Chain $X_1 \rightarrow X_2 \rightarrow X_3 \rightarrow X_4$; because \textit{CIM} satisfies DPI, indirect interactions represented by the red arrows will be removed as edges in the network discovery algorithm for both ARACNe and MRNET.  (b) Results MRNET applied to SYNTREN300 dataset with a global noise level of $10$ and a local noise level of $5$, using \textit{CIM} and various estimators of mutual information (\textit{MI}), for 200 Monte-Carlo simulations.  The median performance of \textit{CIM} exceeds the next best estimator of \textit{MI}, the \textit{KNN20} by $11.77\%$, which corresponds to accurate detection of $55$ more edges from the true network.}
\end{figure}

We conclude by recognizing that although generalizations about all datasets cannot be drawn from these findings, it is prudent to understand the details of the dataset being analyzed.  More specifically, in the context of assessing dependency and modeling joint behavior probabilistically, the simulations conducted show the importance of understanding whether dependencies are monotonic or not.

\subsubsection{Markov Network Modeling}
An immediate implication of \textit{CIM} satisfying the DPI, from Section \ref{sec:dpi}, is that it can be used for network modeling and information flow of data through Markov chains.  This is done by performing repeated Max-Relevance Min-Redundancy (\textit{mRMR}) feature selection \cite{mrmr} for each variable in the dataset and construct a Maximum Relevancy Network (MRNET) \cite{mrnet}.  In principle, for a random variable $X_j \in X$, \textit{mRMR} works by ranking a set of predictor variables $X_{S_j} \subseteq \{X \setminus X_j\}$ according to the difference between the mutual information (\textit{MI}) of $X_i \in X_{S_j}$ with $X_j$ (the relevance) and the average \textit{MI} with the selected variables in $X_{S_j}$ (the redundancy).  By choosing the variable that maximizes this difference, a network can be constructed in which direct interactions between variables imply edges.  By virtue of Theorem \ref{thm:cim_dpi}, \textit{CIM} and average \textit{CIM} can be substituted for the \textit{MI} and the average \textit{MI} to apply \textit{CIM} to the MRNET reconstruction.  As for Fig.~\ref{fig:network}, using the MRNET algorithm and Theorem \ref{thm:cim_dpi}, we can readily say that

\begin{align*}
CIM(X_1,X_2) - & 0.5[CIM(X_1, X_3) + CIM(X_1, X_4)] \geq \\ 
& CIM(X_1,X_3) - 0.5[CIM(X_1, X_2) + CIM(X_1, X_4)],
\end{align*}
and 
\begin{align*}
CIM(X_1,X_2) - & 0.5[CIM(X_1, X_3) + CIM(X_1, X_4)] \geq \\ 
& CIM(X_1,X_4) - 0.5[CIM(X_1,X_2) + CIM(X_1,X_3)],
\end{align*}
yielding the connection between $X_1$ and $X_2$ in Fig.~\ref{fig:network}.  Similar reasoning can be applied to the other network connections.  Simulation results discussed in Section \ref{sec:synth_experiments} motivate the use of \textit{CIM} as a substitute for the \textit{MI}.  In that section, we compare the statistical power of \textit{CIM} to various estimators of the \textit{MI} including: 1) k-nearest neighbors (\textit{k-NN}) estimation \cite{knn_mi}, 2) adaptive partitioning (\textit{AP}) MI estimation \cite{shannonapMI}, and 3) MI estimation via von Mises expansion (\textit{vME}) \cite{vonMisesMI}, and show that \textit{CIM} is more powerful.  This suggests that \textit{CIM} is indeed a viable alternative for use in the estimation of Markov networks from datasets.  


We explore the utility of \textit{CIM} for Markov network modeling in the domain of computational biology by using the \textit{MRNET} algorithm with \textit{CIM} and the \textit{MI} estimators previously described using the gene regulatory network benchmarking tool \textbf{netbenchmark}. That tool uses over 50 datasets of known gene regulatory networks and compares the performance of a provided algorithm when different amounts of noise are added to the datasets in order to assess in a standardized way, the performance of \textit{MI} based network reconstruction algorithms \cite{netbenchmark_cite}.  The datasets used by \textbf{netbenchmark} are different than the gene expression datasets we previously analyzed for monotonicity.  The area under the precision-recall curve of the 20 most confident predictions (AUPR20) is shown for \textit{MRNET} in Fig.~\ref{fig:mrnet_all_datasets} using $CIM$ and the various estimators of the \textit{MI}, for a global noise level of $10$ and a local noise level of $5$.  The results reveal that for the 200 different variations of the \textbf{syntren300} dataset that were compared, the median performance of the \textit{MRNET} is greater when using \textit{CIM} by $11.77\%$, which corresponds to accurate detection of $55$ more edges from the true network.  Although not shown here, for a sweep of both global and local noise levels between $10$ and $50$, \textit{CIM} consistently showed greater performance.  On average, \textit{CIM} was able to discover $5.18\%$ more edges over these noise ranges, which corresponds to $24$ more edges in the \textbf{syntren300} network.  These results are not surprising, and are corroborated by the analysis and the curves displayed in Fig.~\ref{fig:cim_vs_ite}.  


\section{Conclusions and Future Work} \label{sec:conclusion}

In this paper, we have introduced a new statistic of dependence between discrete, hybrid, and continuous random variables and stochastic signals termed \textit{CIM}.  We showed that this index follows most of R\'enyi's properties for a metric of dependence, satisfies the DPI, and is self-equitable.  The implications of satisfying the DPI are discussed in the context of the Markov network construction using the DPI measures.  \textit{CIM} is then compared to other measures of mutual information and state-of-the-art nonparametric measures of dependence.  It is shown to compare favorably and similarly to these compared metrics, respectively, in various synthetic data experiments.  A unique output of \textit{CIM} estimation algorithm, the identification of the regions of monotonicity in the dependence structure, is used to analyze numerous real world datasets.  The results reveal that among all the datasets compared, at least 96\% of the statistically significant dependencies are indeed monotonic.  The simulations highlight the need to fully understand the dependence structure before applying statistical techniques.  

While \textit{CIM} is a powerful tool for bivariate data analysis, there are many directions to further this research.  A logical first step is to extend \textit{CIM} to a measure of multivariate dependence.  
Additional research can be conducted to improve the performance of \textit{CIM} algorithm for monotonic dependencies, as this is an important class of dependencies.  Another area of research is to extend \textit{CIM} to a measure of conditional dependence.  By the invariance property of copulas to strictly increasing transforms \cite{Embrechts01}, we can readily state that if $\{\mathbf{Y} \bigCI \mathbf{Z}\} | \mathbf{X}$, then $\{\mathbf{V} \bigCI \mathbf{W}\} | \mathbf{U}$, where $\mathbf{U}=(U_1, \dots, U_d) = (F_{X_1}(x_1), \dots, F_{X_d}(x_d))$, $\mathbf{V}=(V_1, \dots, V_k) = (F_{Y_1}(y_1), \dots, F_{Y_k}(y_k))$, and $\mathbf{W}=(W_1, \dots, W_n) = (F_{Z_1}(z_1), \dots, F_{Z_k}(z_n))$, and $\mathbf{X}, \mathbf{Y}$, and $\mathbf{Z}$ are random vectors of arbitrary dimensionality.  Due to the invariance property, conditional independence (and dependence) can be measured with the pseudo-observations by borrowing techniques from partial correlation.  Initial results have shown promising results for this application of \textit{CIM}.

\section*{Acknowledgments} \label{sec:ack}
We would like to acknowledge the Hume Center at Virginia Tech for its support.  We also graciously thank Vinodh N. Rajapakse, from the National Institutes of Health (NIH), and Scott Novotney and Stephen Rawls from USC's Information Sciences Institute for providing valuable feedback to improve the quality of the paper.  Finally, we are sincerely grateful to the anonymous referee's reviews, which have greatly improved the accuracy and quality of the manuscript.

\pagebreak
\appendix

\section{\textit{CIM} estimation Algorithm}\label{appendix_cim_algo}
\begin{breakablealgorithm}
	\caption{\textit{CIM}}\label{alg:cim}
	
	\begin{algorithmic}[1]
		\Function{compute-cim}{$msi$, $\alpha$}
		\State $\mathbf{si} \gets [1,\frac{1}{2},\frac{1}{4},\frac{1}{8},\dots,\frac{1}{msi}]$ \Comment{Scanning increments to be tested}
		\State $\mathbf{uv_{cfg}} \gets [\varA{u-v},\varA{v-u}]$ \Comment{Orientations of data to be tested}
		\State $m_{max} \gets 0, \mathbf{RR} \gets []$
		\For{$uv_{cfg}$ in $\mathbf{uv_{cfg}}$}
		\For{$si$ in $\mathbf{si}$}
		\State $\boldsymbol{\tau}, \mathbf{R} \gets \textproc{scan-unit-sq}(si,uv_{cfg},\alpha)$
		\State $m \gets 0$
		\ForAll{$\boldsymbol{\tau}, \mathbf{R}$} \Comment{Compute (\ref{cimeq}) for detected regions}
		\State $n_R \gets \textproc{getNumPoints}(\mathbf{R})$ \footnote{gets the number of points encompassed by the region $\mathbf{R}$}
		\State $m \gets m + \frac{n_R}{n} \tau_R$
		\EndFor
		\If{$m > m_{max}$} \Comment{Maximize (\ref{cimeq}) over all scanning increments}
		\State $m_{max} \gets m$
		\State $RR \gets R$
		\EndIf
		\EndFor
		\EndFor \\
		\Return $m_{max},\mathbf{RR}$
		\EndFunction
		
		\Function{scan-unit-sq}{$si,uv_{cfg},\alpha$}
		\State $\mathbf{R} \gets \textproc{createNewRegion}$ \footnote{creates a new region of monotonicity from the boundary where the previous region was determined to end}
		\State $\mathbf{RR} \gets []$
		\While{uniqSqNotCovered} \footnote{a variable which flags when the expansion of $\textbf{R}$ is covering the entire unit square.}
		\State $\mathbf{R} \gets \textproc{expandRegion}(si, uv_{cfg})$  \footnote{expands the region by the scanning increment amount, $si$, as depicted in Fig.~\ref{fig:cim_algo_explain} in the orientation specified by the $uu_{cfg}$}
		\State $m \gets |\hat{\tau}_{KL}(\mathbf{R})|$ \Comment{$|\hat{\tau}_{KL}|$ of the points encompassed by $\mathbf{R}$}
		\State $n_R \gets \textproc{getNumPoints}(\mathbf{R})$
		\State $\sigma_C \gets 4(1-\hat{\tau}_{KL}(\mathbf{R})^2)$ \Comment{Hypothesis test detection threshold}
		\If{$\neg \textproc{newRegion}(\mathbf{R})$} \footnote{determines if the region $\textbf{R}$ was created in the last loop iteration or not}
		\If{$m < (m_{prev}-\frac{\sigma_C}{\sqrt{n_R}} u_{1-\frac{\alpha}{2}})$}
		\State $\mathbf{RR} \gets \textproc{storeRegion}(\mathbf{R})$ \footnote{called when a boundary between regions is detected; stores the region $\textbf{R}$'s boundaries and the value of $|\tau_{KL}|$ for this region.}
		\State $\boldsymbol{m} \gets m $
		\State $\mathbf{R} \gets \textproc{createNewRegion}(uv_{cfg})$
		\EndIf
		\EndIf
		\State $m_{prev} \gets m$
		\EndWhile \\
		\Return $\boldsymbol{m}, \mathbf{RR}$
		\EndFunction
		
	\end{algorithmic}
\end{breakablealgorithm}

\section{Streaming $\tau$ Algorithm}\label{alg2_appendix}
\begin{breakablealgorithm}
	\caption{$\tau^S_{KL}$}\label{alg:taukl}
	\begin{algorithmic}[1]
		\Function{consume}{}
		\State $ii_{end} \gets ii_{end} + 1$
		\State $mm \gets mm + 1$ \Comment{Increment number of samples, m, we have processed}
		\State $mmc2 \gets mmc2 + mm - 1$  \Comment{Increment running value of ${m \choose 2}$ for denominator}
		\LeftComment{Get the subset of \textbf{u} and \textbf{v}}
		\State $\mathbf{u'} \gets \mathbf{u}(ii_{begin}:ii_{end})$, $\mathbf{v'} \gets \mathbf{v}(ii_{begin}:ii_{end})$
		\LeftComment{Compute ordering of new sample, in relation to processed samples}
		\State $\Delta \mathbf{u} \gets \mathbf{u'}(end)-\mathbf{u'}(end-1:-1:1)$ 
		\State $\Delta \mathbf{v} \gets \mathbf{v'}(end)-\mathbf{v'}(end-1:-1:1)$ 
		\State $u^+ \gets \Sigma \left[ \mathbbm{1}{\left( \Delta \mathbf{u} > 0 \cap \Delta \mathbf{v} \neq 0 \right) } \right], u^- \gets \Sigma \left[ \mathbbm{1}{\left( \Delta \mathbf{u} < 0 \cap \Delta \mathbf{v} \neq 0 \right) } \right]$
		\State $v^+ \gets \Sigma \left[ \mathbbm{1}{ \left( \Delta \mathbf{v} > 0 \cap \Delta \mathbf{u} \neq 0 \right) } \right], v^- \gets \Sigma \left[ \mathbbm{1}{\left( \Delta  \mathbf{v} < 0 \cap \Delta \mathbf{u} \neq 0 \right) } \right]$
		\LeftComment{Compute the running numerator, $K$, of $\tau_{KL}$}
		\If{$u^+ < u^- $}
		\State $kk \gets v^- - v^+$
		\Else
		\State $kk \gets v^+ - v^-$
		\EndIf
		\State $K \gets K + kk$
		\LeftComment{Count number of times values in $u$ and $v$ repeat}
		\State $\textbf{uMap}(\mathbf{u'}(end)) \gets \textbf{uMap}(\mathbf{u'}(end)) + 1, uu \gets uu + \textbf{uMap}(\mathbf{u'}(end)) - 1$
		\State $\textbf{vMap}(\mathbf{v'}(end)) \gets \textbf{vMap}(\mathbf{v'}(end)) + 1, vv \gets vv + \textbf{vMap}(\mathbf{v'}(end)) - 1$
		\LeftComment{Compute threshold for determining if data is hybrid via a threshold heuristic}
		\If{$\neg \bmod(mm, OOCTZT)$}
		\State $mmG \gets mmG + 1$, $ctzt \gets ctzt + mmG  - 1$
		\EndIf
		\State $uuCtz \gets (uu \leq ctzt)$, $vvCtz \gets (vv \leq ctzt)$
		\LeftComment{Compute the denominator of $\tau_{KL}$ depending on whether data was hybrid or not}
		\If{$(uuCtz \cap vv>0) \cup (vvCtz \cap uu>0)$}
		\State $tt \gets \textbf{max}(uu,vv)$
		\State $den \gets \sqrt{mmc2-tt} \sqrt{mmc2-tt}$
		\Else
		\State $den \gets \sqrt{mmc2-uu} \sqrt{mmc2-vv}$
		\EndIf
		\If{$K==0 \cap den==0$}
		\State $\tau_{KL} = 0$
		\Else
		\State $\tau_{KL} = \frac{K}{den}$
		\EndIf
		\Return $\tau_{KL}$
		\EndFunction
		
	\end{algorithmic}
\end{breakablealgorithm}

\section{Real-world Data Experiments Details}\label{appendix:datadetails}
Real-world data analyzed for the monotonicity results shown above in Section \ref{sec:cim_real_data} was derived from online sources.  

\subsection{Gene Expression Data}
The gene expression related data was downloaded from the Broad Institute at the URL: \url{http://portals.broadinstitute.org/cgi-bin/cancer/datasets.cgi}.  The enumeration below lists the specific files which were downloaded from the URL provided above (all with the .gct extension).

\begin{enumerate}
	\begin{multicols}{3}
		\item ALL
		\item beer\_lung\_for\_p53
		\item Breast\_A
		\item Breast\_B
		\item Children\_NE
		\item Common\_miRNA
		\item crash\_and\_burn
		\item DLBCL\_A
		\item DLBCL\_B
		\item DLBCL\_C
		\item DLBCL\_D
		\item Erythroid
		\item GCM\_All
		\item glioma\_classic\_hist
		\item glioma\_nutt\_combo
		\item hep\_japan
		\item HL60
		\item HSC\_FDR002
		\item Iressa\_Patient1\_ams
		\item leuGMP
		\item leukemia.top1000
		\item lung\_datasetB\_outcome
		\item LungA\_1000genes
		\item met
		\item miGCM\_218
		\item MLL\_AF9
		\item mLung
		\item Multi\_A
		\item Multi\_B
		\item Normals\_Leu
		\item Novartis\_BPLC.top1000
		\item PDT\_miRNA
		\item Rap3hour\_control
		\item Rap24hour\_control
		\item Res\_p0005
		\item Sens\_p001
		\item Sens\_p0005
	\end{multicols}
	
	\begin{multicols}{2}
		\item medullo\_datasetC\_outcome
		\item lung\_annarbor\_outcome\_only
		\item med\_macdonald\_from\_childrens
		\item megamiR\_data.normalized.log2.th6
	\end{multicols}
	\item Myeloid\_Screen1\_newData\_021203\_ams.AML\_poly\_mono
	\item Sanger\_Cell\_Line\_project\_Affymetrix\_QCed\_Data\_n798
	
\end{enumerate}
The files, natively in GCT format, were stripped of metadata and converted to CSV files, and scanned for significant dependencies ($\alpha < 0.05$).  The number of regions for the significant dependencies were then counted to determine the number of monotonic regions.  The script to perform the conversion from GCT to CSV is provided at: \url{https://github.com/stochasticresearch/depmeas/tree/master/test/python/gcttocsv.py}.  Additionally, the Matlab scripts to process the pairwise dependencies and produce the monotonicity results is provided at: \url{https://github.com/stochasticresearch/depmeas/tree/master/test/analyze_cancerdata.m}.  

\subsection{Financial Returns Data}
The financial returns related data was downloaded from both \url{finance.yahoo.com} and \url{Investing.com}.  We query the web API of these websites to download all available historical data (from Jan 1985 - Jan 2017) for the following indices:

\begin{enumerate}
	\begin{multicols}{4}
		
		\item A50
		\item AEX
		\item AXJO
		\item BFX
		\item BSESN
		\item BVSP
		\item CSE
		\item DJI
		\item FCHI
		\item FTSE
		\item GDAXI
		\item GSPC
		\item HSI
		\item IBEX
		\item ITMIB40
		\item IXIC
		\item JKSE
		\item KOSPI
		\item KSE
		\item MICEX
		\item MXX
		\item NK225
		\item NSEI
		\item OMXC20
		\item OMXS
		\item PSI20
		\item SETI
		\item SPTSX
		\item SSEC
		\item SSMI
		\item STOXX50E
		\item TA25
		\item TRC50
		\item TWII
		\item US2000
		\item XU100
		
	\end{multicols}
\end{enumerate}

Less than 1\% of the downloaded data was missing.  In order to ease processing, missing data fields were imputed with the last known index price.  The first difference of the stock prices was calculated in order to derive the returns data.  The returns data was first determined to be stationary by the Dickey-Fuller test.  After these procedures, pairwise dependencies between coherently aligned time series were computed.  Because different amounts of historical data were available for the various indices, only the subset of data which belonged to both time series was tested for a significant dependency.  

The script to perform the missing data imputation and raw data normalization is provided at: \url{https://github.com/stochasticresearch/depmeas/tree/master/test/python/normalizeStocksFiles.py}.  Additionally, the Matlab scripts to process the pairwise dependencies and produce the monotonicity results is provided at: \url{https://github.com/stochasticresearch/depmeas/tree/master/test/analyze_stocksdata.m}.  Finally, the raw stocks data is provided at \url{https://figshare.com/articles/Stocks_Data/4620325}.

\subsection{Climate Data}
The climate data was downloaded from the following links: 

\begin{enumerate}
	\item \url{https://www.kaggle.com/uciml/el-nino-dataset}
	\item \url{https://www.kaggle.com/sogun3/uspollution}
	\item \url{https://tinyurl.com/berkeleyearth}
\end{enumerate}

The El-Nino data was normalized by extracting the zonal winds, meridional winds, humidity, air temperature, and sea surface temperature data from the dataset.  The code to extract these features, and all to be described features from other climate related datasets is provided at: \url{https://github.com/stochasticresearch/depmeas/tree/master/test/python/normalizeClimateFiles.py}.  Because these datapoints were collected over multiple decades and large chunks of missing data existed, each chunk of contiguous data (with respect to time) was analyzed separately.  The script to identify these chunks and coherently compute pairwise dependencies after checking for stationarity is provided at: \url{https://github.com/stochasticresearch/depmeas/tree/master/test/analyze_elnino.m}.  

The global land temperatures data was normalized by extracting the land temperature for each country over the available date ranges.  Again, due to significant chunks of missing data, each chunk of contiguous data was analyzed separately.  The script to identify these chunks and coherently compute pairwise dependencies after checking for stationarity is provided at: \url{https://github.com/stochasticresearch/depmeas/tree/master/test/analyze_landtemperatures.m}.

The US pollution data was normalized by extracting NO\textsubscript{2} Air Quality Indicators (AQI), O\textsubscript{3} AQI, SO\textsubscript{2} AQI, and CO AQI for each location over the available date ranges.  Again, due to significant chunks of missing data, each chunk of contiguous data was analyzed separately.  The script to identify these chunks and coherently compute pairwise dependencies after checking for stationarity is provided at: \url{https://github.com/stochasticresearch/depmeas/tree/master/test/analyze_pollution.m}.  

\bibliographystyle{elsarticle-num-names}

\end{document}